\documentclass[lettersize,journal,10pt]{IEEEtran}

\usepackage{amsmath,amsfonts}
\usepackage{algorithmic}
\usepackage{array}
\usepackage[font=normalsize]{subfig}
\usepackage{textcomp}
\usepackage{stfloats}
\usepackage{url}
\usepackage{verbatim}
\usepackage{graphicx}
\def\BibTeX{{\rm B\kern-.05em{\sc i\kern-.025em b}\kern-.08em
    T\kern-.1667em\lower.7ex\hbox{E}\kern-.125emX}}
\usepackage{balance}
\usepackage[utf8]{inputenc} 
\usepackage[T1]{fontenc}    
\usepackage{hyperref}       
\usepackage{url}            
\usepackage{booktabs}       
\usepackage{amsfonts}       
\usepackage{nicefrac}       
\usepackage{microtype}      
\usepackage{graphicx} 
\usepackage{symbolDef}
\usepackage{xcolor}
\usepackage{pifont}

\usepackage{multirow}
\usepackage{cleveref}
\usepackage{amsmath}
\usepackage{listings}
\usepackage{enumitem}
\usepackage{amsthm}
\usepackage{bbm}
\usepackage{color}
\usepackage{wrapfig}

\newcommand{\changed}[1]{{\color{black} #1}}

\newcommand{\first}[1]{\textbf{#1}}
\newcommand{\second}[1]{\textit{\underline{#1}}}

\newtheorem{theorem}{Theorem}

\newtheorem{proposition}{Proposition}
\newtheorem{lemma}{Lemma}
\newtheorem{definition}{Definition}

\newtheorem{remark}{Remark}

\Crefname{equation}{Eq.}{Eqs.}
\Crefname{figure}{Fig.}{Figs.}
\Crefname{tabular}{Tab.}{Tabs.}
\Crefname{theorem}{Thm.}{Thms.}
\Crefname{definition}{Def.}{Defs.}
\Crefname{section}{Sec.}{Secs.}

\begin{document}
\title{Sparse Covariance Neural Networks}
\author{Andrea Cavallo, Zhan Gao, Elvin Isufi \\
\thanks{Part of this work was supported by the TU Delft AI Labs programme, the NWO OTP GraSPA proposal \#19497, and the NWO VENI proposal 222.032. A. Cavallo and E. Isufi are with the Delft University of Technology, Delft, Netherlands. Z. Gao is with the University of Cambridge, Cambridge, UK.
        Emails:  
        \href{mailto:a.cavallo@tudelft.nl}{a.cavallo@tudelft.nl}, 
        \href{mailto:zg292@cam.ac.uk}{zg292@cam.ac.uk}, 
        \href{mailto:e.isufi-1@tudelft.nl}{e.isufi-1@tudelft.nl}  . (Corresponding author: Zhan Gao.)}} 


\maketitle

\begin{abstract}
Covariance Neural Networks (VNNs) perform graph convolutions on the covariance matrix of 
input data to leverage correlation information as pairwise connections. They have achieved success in a multitude of applications such as 
neuroscience, financial forecasting, and sensor networks. 
However, the empirical covariance matrix on which VNNs operate typically contains spurious correlations, creating a mismatch with the actual covariance matrix that degrades VNNs' performance and 
computational efficiency.
To tackle this issue, we put forth Sparse coVariance Neural Networks (S-VNNs), a framework that applies sparsification techniques on the sample covariance matrix and incorporates the latter into the VNN architecture. We investigate the S-VNN when the underlying data covariance matrix is both sparse and dense. When the true covariance matrix is sparse, we propose hard and soft thresholding to improve the covariance estimation and reduce the computational cost. Instead, when the true covariance is 
dense, we propose a stochastic sparsification where data correlations are dropped in probability according to principled strategies. Besides performance and computation improvements, we show that S-VNNs are more stable to finite-sample covariance estimations than nominal VNNs and the analogous sparse principal component analysis. By analyzing the impact of sparsification on their behavior, we tie the S-VNN stability to the data distribution and sparsification approach.
We support our theoretical findings with experimental results on a variety of application scenarios, ranging from brain data to human action recognition, and show an improved task performance, improved stability, and reduced computational time compared to alternatives. 
\end{abstract}


\section{Introduction}

Covariance-based data processing is key to signal processing and machine learning pipelines due to its ability to whiten data distributions, identify principal directions, and estimate the interdependencies among features. Such advantages have been shown in several applications including estimating brain connectivity~\cite{bessadok2022graph,QIAO2016399}, financial data~\cite{cardoso2020algorithms,wang2022network} and human action recognition~\cite{liao2022har,wang2023mhagnn}. One prominent method built upon the covariance matrix is the Principal Component Analysis (PCA), which maximizes the variance of data points by projecting them onto the eigenvectors of the said matrix~\cite{Jolliffe2016PrincipalCA}.
However, PCA is unstable to errors in the covariance estimation, i.e., a poor estimate of the covariance matrix and its eigenvectors may lead to unpredictably bad results, especially in \changed{regimes where the data dimension and number of samples are of the same order}
or when the covariance matrix eigenvalues are close to each other~\cite{Jolliffe2002pca,paul2007asymptotics,baik2005phase}, and this limitation extends to advanced PCA variants such as kernel PCA~\cite{zwald2005convergence}. 
To overcome this issue, coVariance Neural Networks (VNNs) have been proposed~\cite{sihag2022covariance}. In essence, VNNs treat the covariance matrix as a graph with each variable as a node and the covariance value as an edge weight, and apply graph-based convolutional learning on this graph. Such an operation performs hierarchical spectral learning on the covariance graph, which ties to PCA as it learns the importance of the principal directions of the data for the task at hand.
%
%
Due to the graph convolution operation, VNNs inherit the stability property of graph neural networks (GNNs)~\cite{gama2020stability}, which makes them stable to finite-data estimation errors in the covariance matrix; i.e., the output difference of a VNN operating on the sample covariance (estimated from data samples) and the true covariance is bounded~\cite{sihag2022covariance,sihag2024explainable}. Additionally, they enjoy transferability across covariances of different dimensions~\cite{sihag2023transferablility}.
%
%
These properties make VNNs effective in a variety of settings ranging from datasets with different resolutions~\cite{sihag2023transferablility} to temporal~\cite{cavallo2024stvnn} and biased data~\cite{cavallo2024fairvnn}. Furthermore, VNNs represent a general framework for GNNs on graphs estimated from correlations which is a common practice in fields like brain connectivity estimation~\cite{bessadok2022graph,QIAO2016399}, financial data processing~\cite{cardoso2020algorithms,wang2022network} and human action recognition~\cite{liao2022har,wang2023mhagnn}.

Nevertheless, VNNs operate on the sample covariance matrix, which results in 
two limitations. First, VNNs are sensitive to settings in which the true covariance is sparse, especially in regimes where \changed{the data dimension and number of samples are of the same order}~\cite{paul2007asymptotics,baik2005phase,F_ral_2007}. Second, VNNs are computationally expensive and memory inefficient, as the sample covariance matrix is typically dense due to estimation errors, regardless of the true covariance matrix being dense or sparse, which is an emphasized computational burden in high-dimensional datasets.

To overcome these limitations, we study the effect of sparse covariance matrix estimators on the learned representation of the VNN, and how properties of the latter affect and are affected by the sparsification strategy as well as the data distribution. More specifically, we propose Sparse coVariance Neural Networks (S-VNNs), which incorporate sparsification-based covariance regularizers into the VNN architecture, and analyze the impact of sparsification on performance when both the true covariance matrix is sparse and dense.

\textit{When true covariance matrix is sparse}, we incorporate thresholded covariance estimates into S-VNNs and show that they are robust to the finite-data estimation error. For the more generic scenario where \textit{the true covariance matrix is not necessarily sparse}, we put forth a stochastic sparsification approach and discuss its impact on the S-VNN. 
Our analysis reveals an inherent trade-off between stability, sparsification degree and strategy, which, in turn, translates into a trade-off between finite-data performance and computational efficiency. We show that dropping covariances according to their absolute values leads to a higher stability but low sparsification, whereas dropping a fixed percentage of covariance values improves sparsification at the expense of stability. We corroborate our findings with experiments on synthetic and real datasets, highlighting the benefits of sparsity on covariance-based neural processing architectures in terms of both performance and computation. 
Our specific contributions are as follows:

\noindent \textbf{(C1) Sparse covariance neural networks.} We develop sparse covariance neural networks (S-VNNs) that perform sparsification techniques on the sample covariance matrix for performance improvement and computation reduction. We study the impact of sparsification on the learned embeddings and reveal a trade-off between sparsity, performance, and computational complexity.

\noindent \textbf{(C2) Sparsification techniques.} We consider different sparsification techniques for S-VNNs. When the true covariance is sparse, we apply hard and soft thresholding; when the true covariance is dense, we propose a stochastic sparsification strategy that allows controlling the desired sparsification level in line with the found trade-off in (C1).

\noindent \textbf{(C3) Stability analysis.}  
We investigate the role of different sparsification techniques and characterize the effects of covariance sparsification on the VNN performance through a finite-data sample stability analysis. 
The results show that S-VNNs reduce computational cost and also improve stability of the learned embeddings to covariance estimation errors, compared to the nominal VNN and the sparse PCA.  

\noindent \textbf{(C4) Empirical validation.} We validate our theoretical findings with experiments on both synthetic and real data.
The results demonstrate the better performance of covariance sparsification in low-data settings and show the improved efficiency on brain data and human action recognition use cases. We make our code available\footnote{\url{https://github.com/andrea-cavallo-98/SVNN}}.

The rest of this paper is organized as follows. Section~\ref{sec:problem_formulation} introduces VNNs, their link to PCA and states our problem setting. Sections~\ref{sec:sparse_true_covariance} and~\ref{sec:generic_true_cov} present our sparse covariance neural networks, propose different sparsification strategies when the true covariance is sparse or generic, and characterize their corresponding effects on VNN's stability. Finally, Section~\ref{sec:numerical_results} presents the experimental results.

\section{Preliminaries}
\label{sec:problem_formulation}

Consider $t$ samples $\{\vcx_i\}_{i=1}^t$ of a random vector $\vcx \in \mathbb{R}^N$ with mean $\vcmu = \mathbb{E}[\vcx] \in \mathbb{R}^N$ and finite covariance $\mtC = \mathbb{E}[(\vcx-\vcmu)(\vcx-\vcmu)^\Tr] \in \mathbb{R}^{N \times N}$.
We can estimate these quantities from $t$ observed samples as $\vchmu = \sum_{i=1}^t\vcx_i/t$, $\mthC = \sum_{i=1}^t(\vcx_i-\vchmu)(\vcx_i-\vchmu)^\Tr/t$. 
The estimated covariance $\mthC$ admits the eigendecomposition $\mthC = \mthV\mathbf{\hat{\Lambda}}\mthV^\Tr$ with orthonormal eigenvectors $\mthV = [\vchv_1, \dots, \vchv_{N}]\in \mathbb{R}^{N\times N}$ and eigenvalues $\mathbf{\hat{\Lambda}} = \text{diag}(\schlambda_1, \dots, \schlambda_{N})$.
Using 
the sample covariance $\mthC$ and its eigendecomposition for data processing introduces errors related to the statistical uncertainty of the estimation, i.e., the sample covariance $\mthC$ can significantly differ from the true covariance $\mtC$ if the number of samples $t$ is small, the dimensionality $N$ is large or the eigenvalues of the true covariance are close to each other~\cite{Jolliffe2002pca}. This creates problems for techniques that model data interdependencies through their covariance, such as PCA, as we elaborate next.

\subsection{PCA Transform}

Given a zero-mean dataset $\mtX\in \mathbb{R}^{N\times t}$ (w.l.o.g.) with each column a data sample $\vcx \in \mathbb{R} ^ N$, PCA applies a transformation $\mttX = \mtT ^\Tr \mtX$ with $\mtT \in \mathbb{R}^{N\times N}$ such that the covariance of the transformed data $\mttX \mttX^\Tr$ is diagonal.
This is achieved by setting $\mtT = \mthV$ where $\mthV$ is the eigenvector matrix of the covariance of the data, because  
$\mttX \mttX^\Tr = \mthV^\Tr \mtX \mtX^\Tr \mthV = \mthV^\Tr (\mthV \mathbf{\hat{\Lambda}} \mthV^ \Tr) \mthV = \mathbf{\hat{\Lambda}}$ due to the orthonormality of the covariance eigenvectors.
Since the principal directions maximize the variance of the transformed data, PCA is also used for dimensionality reduction by selecting only the $k$ covariance eigenvectors corresponding to the largest eigenvalues, i.e., $\mttX_{(k)} = [\mthV]^T_{1,\dots,k}\mtX$ where $[\cdot]_{1,\dots,k}$ selects the first $k$ columns.
This operation can be interpreted as the application of an ideal spectral filter on the covariance eigenvalues $h_\textnormal{PCA}(\lambda) = \mathbf{1}[\lambda \geq \lambda_k]$, i.e., $\mttX_{(k)} = [\operatorname{diag}(h_\textnormal{PCA}(\lambda_1),\dots,h_\textnormal{PCA}(\lambda_N))\mthV^T\mtX]_{1,\dots,k}$---see Fig.~\ref{fig:cov_filter}.

However, PCA presents some limitations. First, operating on the eigenvectors of the sample covariance matrix $\mthV$ makes it unstable to estimation errors, which becomes problematic for close eigenvalues and in sparse low-data regimes (e.g., when the data dimensionality is of the same order as the number of data points~\cite{deshp2016sparse}). \changed{This relates to the PCA transform acting as a high-pass filter (cf. Fig.~\ref{fig:cov_filter}): if errors cause a smaller eigenvalue to artificially surpass a dominant one, this forces PCA to select suboptimal principal components, ultimately degrading the projection---see Lemma~\ref{lemma_pca} for a detailed theoretical stability analysis}. 
Second, PCA requires the explicit computation of the covariance eigendecomposition which is of order $\mathcal{O}(N^3)$. 
Third, PCA is neither inductive nor transferable, i.e., small data drifts, new data samples or samples of different dimensionality require the recomputation of the principal components.
Finally, PCA for dimensionality reduction focuses on the directions of largest variance, which might lose relevant information for the downstream task (e.g., if the data contains high-variance noise).
To counteract these limitations, covariance filters and covariance neural networks have been proposed. 

\begin{figure}
    \centering
    \includegraphics[width=1\linewidth]{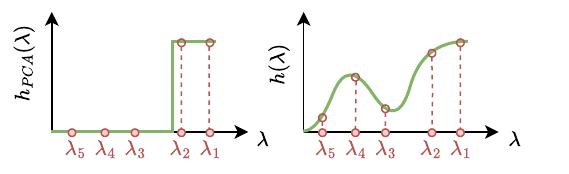}
    \caption{
    Frequency responses for PCA with dimensionality reduction (left) and covariance filters (right). The PCA filter selects the largest $k$ covariance eigenvalues, whereas the covariance filter learns a polynomial filter function.}
    \label{fig:cov_filter}
\end{figure}

\subsection{Covariance Filters}\label{subsec:CF}

\changed{To improve the stability to finite-sample estimation errors, while processing the covariance matrix in an expressive manner, covariance filters~\cite{sihag2022covariance} have been proposed.}
Specifically, each of the $N$ data features is seen as a node of a graph and the data sample $\mathbf{x} \in \mathbb{R}^N$ as a graph signal, where the covariance $\mathbf{C}$ represents the graph structure -- see Fig.~\ref{fig:cov_graph}. 
Covariance filters are graph convolutional filters acting on the covariance graph, i.e., they transform an input signal $\vcx$ into an output $\vcu$ as
\begin{align}
\label{eq:cov_filter}
     \vcu = \mtH(\mthC)\vcx = \sum_{k=0}^Kh_{k}\mthC^k\vcx.
\end{align}
This operation can be computed recursively, i.e., $\mtC^k\vcx = \mtC(\mtC^{k-1}\vcx)$, with a total cost of $\mathcal{O}(kN^2)$ which makes covariance filters more efficient than PCA. Moreover, the term $N^2$ is due to the density of the sample covariance. In this paper, we study sparse covariance estimators which can reduce this cost to $\mathcal{O}(k\|\mthC\|_0)$, where $\|\mthC\|_0$ is the number of non-zero elements of the estimated covariance. 

The behavior of the covariance filter can be characterized by computing the graph Fourier transform of its input and output signal, i.e.,
$\vctu = \mthV^\Tr\vcu = \mthV^\Tr\sum_{k=0}^Kh_k[\mthV\mathbf{\hat{\Lambda}}\mthV^\Tr]^k\vcx = \sum_{k=0}^Kh_k\mathbf{\hat{\Lambda}}^k\mthV^\Tr\vcx$. For the $i$-th entry, this leads to
$\vctu_{i} =  \sum_{k=0}^Kh_k\hat{\lambda}_i^k\vctx_{i} = h(\hat{\lambda}_i)\vctx_{i}$.
That is, the frequency response of the covariance filter is a polynomial $h(\lambda)=\sum_{k=0}^Kh_k\lambda^k$ with frequency variable $\lambda$ specified on the eigenvalues $\hat{\lambda}_i$ -- see Fig.~\ref{fig:cov_filter}. 
The frequency response of the covariance filter highlights its analogy with PCA, since both operations process the eigenvectors of the covariance matrix and there exists a set of coefficients $h_k$ such that the covariance filter recovers the operation of PCA~\cite[Thm. 1]{sihag2022covariance}.
Based on the frequency response, we introduce the definition of Lipschitz covariance filter, which will be instrumental in the following analysis. 
\begin{definition}[Lipschitz covariance filter]
\label{as:lipschitz}
    The covariance filter is Lipschitz with constant $P$ if, for every pair of eigenvalues $\lambda_i,\lambda_j \in [0,\lambda_\textnormal{max}]$, $\lambda_i \neq \lambda_j$, its frequency response satisfies: $|h(\lambda_i)-h(\lambda_j)| \leq  P|\lambda_i - \lambda_j|$, where $\lambda_\textnormal{max}\in[0,\infty)$ identifies a suitable range for the covariance eigenvalues. 
\end{definition}
\changed{The Lipschitz constant $P$ controls the variability of the frequency response of the filter: a larger $P$ allows for sharper variations of $h(\lambda)$, while a smaller $P$ indicates a flatter frequency response.
Building on this property, covariance filters have been shown stable to errors in covariance estimation, with significant improvements over PCA, as we formally discuss next.} To report this result, we introduce the following definitions:
\begin{align}
    k_j = (\mathbb{E}[\|\vcx\vcx^\Tr\vcv_j\|^2]-\lambda_j^2)^{1/2}, \quad k_\textnormal{max} = \max_j k_j, \\
    k_\textnormal{min} = \min_{j,\lambda_j > 0} k_j,\quad \kappa = \max_{i,j:\lambda_i \neq \lambda_j} k_i^2/|\lambda_i - \lambda_j|,
\end{align}
which are terms that relate to the kurtosis of the data distribution (we refer to \cite{sihag2022covariance} for additional details as these are not part of our subsequent analysis).
The stability of covariance filters is stated next. 
\begin{theorem}[\cite{sihag2022covariance}]
\label{th:cov_filter_stability}
    Consider a covariance filter $\mtH(\mtC)$ [cf.~\eqref{eq:cov_filter}] that is Lipschitz with constant $P$ [\Cref{as:lipschitz}]. 
    Consider a true covariance $\mtC$ and its sample estimate from $t$ samples $\mthC$, respectively.
    With probability at least $1-t^{-2\epsilon}- 2\kappa N/t$, for any $\epsilon \in (0, 1/2]$ it holds that 
    \begin{equation*}
    \begin{gathered}
    \label{eq:cov_filter_stability}
    \|\mtH(\mthC) \!-\! \mtH(\mtC)\| \!\leq\! \frac{Pk_\textnormal{max}}{t^{1/2-\epsilon}}\mathcal{O}\left( \!\!\sqrt{N} \!+\! \frac{\|\mtC\|\sqrt{\log (Nt)}}{k_\textnormal{min} t^{2\epsilon}} \!\right)\!\!=\!\beta\!
    \end{gathered}
    \end{equation*}
    for $\|\cdot\|$ denoting the $2-$ or the spectral norm. 
\end{theorem}

We make the constant $k_\textnormal{max}$ explicit in the bound of \Cref{th:cov_filter_stability}, whereas in~\cite[Thm. 2]{sihag2022covariance} it is embedded in the assumption on the filter frequency response. We opt for this formulation for an easier comparison with our following analysis.

This notion of stability follows an extensive line of research on stability of GNNs to generic graph perturbations~\cite{gama2020stability,keriven2020convergence,Ruiz_2020, gao2023learning,cervino2022training,arghal2022robust}, but it specifically considers covariance estimation errors, which is fundamental for covariance filters. 
Since in practical applications the true covariance is not available and covariance filters operate on an estimate, this notion of stability acts as a certificate of performance guarantee w.r.t. the ideal scenario when working with the true covariance matrix and characterizes the impact of data and model characteristics. 
More in detail, it identifies that the output difference of a covariance filter caused by the sample estimation error is bounded proportionally by the square root of the number of samples $t$. The bound increases with the Lipschitz constant $P$, since a higher $P$ allows the frequency response to generate different outputs at close eigenvalues, thus improving the filter's discriminability, but leads to larger sensitivity to estimation errors, thus degrading its stability. 

Covariance filters show increased stability w.r.t. that of PCA, which we investigate in the following lemma. 
\begin{lemma}
\label{lemma_pca}
Consider the PCA with a true covariance matrix $\mtC$ and a sample covariance estimate $\mthC$ of respective eigendecompositions $\mtC = \mtV\mtLambda\mtV^\Tr$ and $\mthC = \mthV\mathbf{\hat{\Lambda}}\mthV^\Tr$.
Then, for any signal $\vcx$ with $\|\vcx\|\leq 1$, it holds with probability $1-o(1)$ that
$\|\mtV^\Tr\vcx-\mthV^\Tr\vcx\| \leq \mathcal{O}(t^{-1/2}(\min_{j}|\lambda_j-\lambda_{j+1}|)^{-1})$.
\end{lemma}
We refer to~\cite{vershynin2018high} for details on the probability $1-o(1)$, which gets closer to 1 as $t$ increases.
\Cref{lemma_pca} shows that PCA stability is inversely proportional to the smallest gap between covariance eigenvalues, leading to unstable behaviors when the eigenvalues are close. Nevertheless, covariance filters do not suffer from this issue as they exhibit a stable response to close eigenvalues, which is modeled by their frequency response and Lipschitz constant $P$.

\begin{figure}
    \centering
    \includegraphics[width=.6\linewidth]{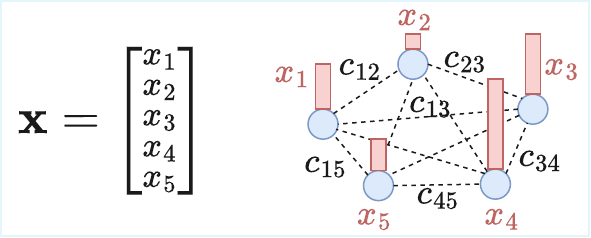}
    \caption{Covariance graph. Each feature is a node, the signal $\vcx$ becomes a graph signal and the edge weight is the covariance value among two features.}
    \label{fig:cov_graph}
\end{figure}

\subsection{Covariance Neural Networks}

A VNN is a sequence of layers and the $l$-th layer assembles a filter bank of $F_{l-1}\times F_{l}$ covariance filters and a non-linear activation function $\sigma(\cdot)$, i.e., 
\begin{align}
\label{eq:vnn_layer}
     \vcu^l_f = \sigma\left(\sum_{g = 1}^{F_{l-1}}\mtH^l_{fg}(\mthC)\vcu^{l-1}_g\right)~f=1,\ldots,F_{l},~l = 1, \ldots, L
\end{align}
where $\{\vcu^{l}_f\in\mathbb{R}^{N}\}_{f=1}^{F_l}$ are the outputs of the $l$-th layer. Each output is produced by the $f$-th covariance filter bank, which contains $F_{l-1}$ covariance filters $\{ \mtH_{fg}^l(\mthC) \}_{g=1}^{F_{l-1}}$ processing each of the signals generated at the previous layer $\{\mathbf{u}_g^{l-1}\in\mathbb{R}^{N}\}_{g=1}^{F_{l-1}}$ separately. 
At the first layer, we have $\{\mathbf{u}_g^0 = \vcx_g\}_{g=1}^{F_0}$ where $F_0$ is the node feature size.
We denote the VNN architecture as $\fnPhi(\vcx, \mthC, \mathcal{H})$, where $\mathcal{H} = \{h_{klfg}\}_{klfg}$ contains all network parameters for each order $k$, layer $l$, input signal $g$ and output signal $f$. 
The output of the last layer $\mathbf{u}^L = \Phi(\mathbf{x}, \mathbf{\hat{C}}, \mathcal{H})$ contains the final representations generated by the VNN and can be directly used for a downstream task (e.g., classification or regression) or further processed by a readout layer.
 The model parameters $\mathcal{H}$ are optimized to minimize a task-specific loss (e.g., cross-entropy for classification or mean squared error for regression) 
 over a training set. 
The following theorem extends the stability of covariance filters to VNNs. 
\begin{theorem}[\cite{sihag2022covariance}]
\label{th:vnn_stability}
    Consider a VNN $\fnPhi(\vcx, \mtC, \mathcal{H})$ of $L$ layers with $F_l = F$ for $l=1,\ldots,L$ and nonlinearities $\sigma(\cdot)$ such that $|\sigma(a)-\sigma(b)|\leq |a-b|$. 
    Let $\beta$ be the largest stability bound of all covariance filters in the VNN as per~\Cref{th:cov_filter_stability}. Consider a generic data sample $\vcx$ with covariance $\mtC$ and $\|\vcx\| \leq 1$ w.l.o.g..
    Then, with probability $1-o(1)$ it holds that 
    \begin{equation*}
    \begin{gathered}
    \label{eq:vnn_stability}
 \|\fnPhi(\vcx, \mtC, \mathcal{H}) \!-\! \fnPhi(\vcx, \mthC, \mathcal{H})\| \!\leq\! LF^{L\!-\!1}\!\beta.
    \end{gathered}
    \end{equation*}
\end{theorem}

With stability in finite-data settings, VNNs have been shown effective in covariance-based learning tasks, both on static and temporal data~\cite{sihag2024explainable,sihag2023transferablility,cavallo2024stvnn,cavallo2024fairvnn}. However, they have two major limitations. 
First, finite-sample covariance estimates often contain spurious correlations~\cite{baik2005phase,paul2007asymptotics}, affecting the performance of the VNN significantly. Second, a VNN on the finite-sample estimate is limited by its quadratic computational complexity in the data dimension, restricting its applicability to low-dimensional settings only.

\subsection{Contributions}
To overcome the above limitations, we propose Sparse VNNs (S-VNNs), which are VNNs operating on sparsified sample covariance matrices.
While PCA-based processing sparsification has been studied~\cite{bickel2008covariance,deshp2016sparse}, its extension to the VNN is challenging because it is unclear: (i) how to perform sparsification appropriately based on the condition of true covariance matrix (i.e., whether the true covariance matrix is sparse or dense); (ii) what is the impact of sparsification on the VNN stability, which calls for an analysis beyond the small perturbation studied in~\cite{sihag2022covariance} since the sparsification perturbation might be large.
We address these aspects by studying different sparsification strategies for S-VNNs in accordance with different conditions of true covariances, and analyze their impact on stability \changed{in the regime $t=\Omega(N)$}.
More in detail, when the true covariance matrix is sparse, we study thresholding-based sparsification strategies for the sample covariance matrix, and show that they result in tighter stability bounds, stronger robustness, and improved computational complexity (Thms.~\ref{cor:sparse_stab_hard}-\ref{cor:vnn_stability_sparsecov}). When the true covariance matrix is dense, we put forth a stochastic sparsification framework in a form akin to dropout (Def.~\ref{def:stochastic_sparsification}), characterize its impact on the VNN stability (Thm.~\ref{th:stability_random}), and propose principled sparsification strategies based on these findings (Sec.~\ref{sec:designsparse}). All proofs are collected in the supplement.

\section{Sparse True Covariance}
\label{sec:sparse_true_covariance}

In applications involving brain and spectroscopic imaging, weather forecasting, financial business \changed{or motion sensors}, the correlations between data points are generally sparse (e.g., only some brain regions activate simultaneously, only some stock prices are affected by similar factors, \changed{only some body parts coordinate to perform specific movements} etc.), leading to sparse true covariance matrices~\cite{rosa2015sparse,fan2011high}. However, the sample covariance matrix is notoriously prone to spurious correlations due to limited sample size, resulting in an inappropriate dense estimates~\cite{Bickel_2008,jobson1980estimation,ledoit2003honey}. 
In this case, 
we leverage thresholding-based sparsification strategies, i.e., hard and soft thresholding, on the sample covariance matrix for S-VNNs and study their effects. 
Compared to other regularized covariance estimations~\cite{ledoit2003honey,Bickel_2008,bien2011sparse,friedman2007lasso}, the hard and soft thresholding provide computational efficiency and theoretical tractability.
\changed{Following previous works on sparse PCA~\cite{bickel2008covariance,deshp2016sparse}, we perform our stability analysis in this section under Gaussian assumptions. While this is more restrictive than the bound in Theorem~\ref{th:cov_filter_stability}, which holds for generic distributions with bounded tails, it allows for a simpler comparison with sparse PCA approaches.}

\subsection{S-VNNs with Hard Thresholding}

Hard thresholding removes the covariance entries below a pre-defined value and it has been shown to improve covariance estimation in settings where \changed{data dimension and number of samples are of the same order}~\cite{bickel2008covariance}.
\begin{definition}[Hard thresholding]
\label{def:hard_thr}
    Given the sample covariance matrix $\mthC$ and a coefficient $\tau>0$, the hard thresholding function is
    $\eta(\mthC)_{ij} = \schc_{ij}$ if $|\schc_{ij}| \geq \tau/\sqrt{t}$, and 0 otherwise.
\end{definition}

Following~\cite{bickel2008covariance}, hard thresholding is inversely dependent on the number of samples as $\tau/\sqrt{t}$. 
This follows the intuition that the non-thresholded estimator approaches the true covariance as $t$ increases, hence, the sparsification is less needed and disappears in the limit of $t \rightarrow \infty$.
Hard thresholding provides a more reliable covariance estimate because it removes small spurious finite-sample errors and, thus, improves the performance of VNNs when the true covariance is sparse. We then characterize the impact of hard thresholding on VNNs with the stability analysis in the following theorem.

\begin{theorem}
\label{cor:sparse_stab_hard}
    Let $\vcx$ be Gaussian with true covariance $\mtC$ belonging to the sparse class $\mathcal{C} = \{\mtC:c_{ii}\leq M,\sum_{j=1}^N \mathbbm{1}[c_{ij} \neq 0]\leq c_0, \forall i\}$ where $M>0$ is a constant, $\mathbbm{1}(\cdot)$ is the indicator function and $c_0$ is the maximum number of non-zero elements in each row of $\mtC$. 
    Let the eigenvalues of the true covariance $\{\lambda_i\}_{i=1}^{N}$ be all distinct and strictly positive. 
    Consider a hard-thresholded sample covariance matrix $\mtbC$ following Def.~\ref{def:hard_thr} with $\tau = M'\sqrt{\log N}$ and $M'$ large enough and a covariance filter $\mtH(\cdot)$ that is Lipschitz with constant $P$. With probability $1-o(1)$, for a large enough constant $C$, it holds that
    \begin{align*}\label{eq:sparse_stab_hard}
        \|\mtH(\mtbC)\vcx-\mtH(\mtC)\vcx\| \leq \\ t^{-1/2}Pc_0\sqrt{N\log{N}}(1+\sqrt{N}) + \mathcal{O}\left(t^{-1}\right).
    \end{align*}
\end{theorem}

\Cref{cor:sparse_stab_hard} shows that hard-thresholded sparse covariance filters --and consequently S-VNN from Thm.~\ref{th:vnn_stability}-- are stable to perturbations on both the finite-data covariance estimation error and the thresholding and that they 
converge to the respective filter and VNN operating over the true covariance at a rate of $1/\sqrt{t}$, as the number of data samples $t$ increases. 
The threshold parameter $\tau$ is set to a value that allows for the recovery of the true covariance sparsity pattern with the desired rate, therefore it does not appear in the bound (see~\cite{bickel2008covariance} for details).
\changed{The bound depends on the sparsity of the true covariance via the term $c_0$: a smaller $c_0$ leads to a lower stability bound, making it tighter for sparser true matrices.}
Importantly, hard thresholding provides a tighter bound than nominal VNNs~\cite{sihag2022covariance} and hard-thresholded sparse PCA~\cite{bickel2008covariance}, as we elaborate next.

\underline{\textit{Comparison with nominal VNN.}} From Thm.~\ref{cor:sparse_stab_hard} and Thm.~\ref{th:vnn_stability}, we see that both bounds decrease with the same order $\mathcal{O}(t^{-1/2})$ (for $\epsilon \rightarrow 0$ in Thm.~\ref{th:vnn_stability}). However, the stability bound of hard-thresholded S-VNN depends linearly on the number of non-zero elements in a row $c_0$, whereas that of the VNN depends on the spectral norm $\|\mtC\|$. For sparse covariance, $c_0 \ll \|\mtC\|$ and the stability bound is tighter.

\underline{\textit{Comparison with sparse PCA.}} 
To provide detailed insights on the stability comparison between sparse covariance filters and sparse PCA, we first analyze the stability of sparse PCA with hard thresholding in the following proposition. 
\begin{proposition}
\label{prop:pca_stab}
    Consider a true covariance matrix $\mtC$ and the thresholded sample estimate $\mtbC$ with respective eigendecompositions $\mtC = \mtV\mtLambda\mtV^\Tr$ and $\mtbC = \mtbV\mathbf{\bar{\Lambda}}\mtbV^\Tr$. Then, for any signal $\vcx$ with $\|\vcx\|\leq 1$, it holds with probability $1-o(1)$ that
    \begin{align}
    \nonumber
        \|\mtV^\Tr\vcx -\mtbV^\Tr\vcx\| \leq t^{-1/2} (\min_j|\lambda_j-\lambda_{j+1}|)^{-1} c_0N\sqrt{2\log{N}}. 
    \end{align}
\end{proposition}

The stability of sparse PCA is inversely proportional to the minimum eigenvalue gap. This term is not present in the stability of S-VNN, i.e., Thm.~\ref{cor:sparse_stab_hard}, as it is absorbed by the filter Lipschitz constant $P$. This indicates that the sparse covariance filter can exhibit a stable spectral behavior for close eigenvalues at the expense of lower discriminability (due to the Lipschitz property [Def. \ref{def:filter_lipschitz}]), but the latter is compensated in the subsequent S-VNN layers that use cascades of filterbanks and non-linearities to 
increase the model expressivity~\cite{Gama_2020,isufi2024graphfilters}. Thus, hard-thresholded S-VNNs attain improved stability compared with sparse PCA in the sparse covariance setting.

\subsection{S-VNNs with Soft Thresholding}
\label{sec:soft_thresholding}

When data follows a spiked covariance model, such as electrocardiogram or brain image data~\cite{johnstone2009sparse}, soft thresholding has been studied to achieve more reliable covariance estimates~\cite{deshp2016sparse}.
Specifically, data points in this model follow $\vcx_i = \sum_{q=1}^r\sqrt{\beta_q}u_{q,i}\vcv_q + \vcz_i$,
where $\vcv_1,\dots,\vcv_r\in\mathbb{R}^N$ are orthonormal vectors with $c_0$ non-zero entries whose magnitudes are lower-bounded by $\theta/\sqrt{c_0}$ for some constant $\theta > 0$, $u_{q,i}\sim\mathcal{N}(0,1)$ and $\vcz_i\sim\mathcal{N}(\mathbf{0},\mtI)$ are independent and follow normal distributions, and $\beta_q\in\mathbb{R}_+$ is a measure of the signal-to-noise ratio. 
In this case, we sparsify the covariance estimate with soft thresholding.
\begin{definition}[Soft thresholding]
\label{def:soft_thr}
    Given the sample covariance matrix $\mthC$ and a coefficient $\tau>0$, we define the soft thresholding function as
    $\eta(\mthC)_{ij} = \schc_{ij} - \textnormal{sign}(\schc_{ij}) \tau/\sqrt{t}$ if $|\schc_{ij}| > \tau/\sqrt{t}$, and 0 otherwise.
\end{definition}
Unlike hard thresholding, which only removes small uncertain values, soft thresholding subtracts a value from all entries to remove uncertainty also from the non-zeroed coefficients.
Again, we set the threshold to decrease with the number of samples $t$ analogously to~\cite{deshp2016sparse} as more accurate covariance estimates reduce the need for sparsification. 
\begin{theorem}
    \label{cor:vnn_stability_sparsecov}
     Consider a soft-thresholded estimate of the covariance matrix $\mtbC$ as per \Cref{def:soft_thr} with $\tau = M'\sqrt{\log(N/c_0^2)}$ and $M',C$ two large enough constants and consider a covariance filter $\mtH(\cdot)$ that is Lipschitz with constant $P$.
     Let the eigenvalues of the true covariance $\{\lambda_i\}_{i=1}^{N}$ be all distinct and strictly positive. Then, the following holds with probability $1-o(1)$:
    \begin{align}
    \nonumber
    \|\mtH(\mtbC)\vcx\!-\!\mtH(\mtC)\vcx\| \!\leq\! \mathcal{O}\left(t^{-1}\!\right)\! + \\ 
 \nonumber t^{-\frac{1}{2}}P\sqrt{N}Cc_0\max(1,\!\lambda_\textnormal{max})\sqrt{\max\left(\log (N/c_0^2),1\right)}(1+\sqrt{N}).
    \end{align}    
\end{theorem}
\Cref{cor:vnn_stability_sparsecov} shows that S-VNNs are stable also when the covariance estimate is soft thresholded and the stability bound decreases with the number of samples at the rate of $t^{-1/2}$. 
The main takeaways for hard thresholded S-VNNs w.r.t. VNNs and sparse PCA hold also for soft-thresholded S-VNNs.
Moreover, soft thresholding achieves a better stability than hard thresholding in low-data settings as the term $\sqrt{\log(N/c_0^2)}$ in Thm.~\ref{cor:vnn_stability_sparsecov} is smaller than the term $\sqrt{\log N}$ in Thm.~\ref{cor:sparse_stab_hard}.  

\smallskip
\noindent\textbf{Computational complexity.}
An S-VNN layer has a computational complexity of order $\mathcal{O}(\|\mtbC\|_0KF_\text{in}F_\text{out})$, where $\|\mtbC\|_0$ is the number of non-zero values of the thresholded sample covariance. This is significantly better than a VNN with a computational complexity of order $\mathcal{O}(N^2KF_\text{in}F_\text{out})$, since $\|\mtbC\|_0$ is generally much smaller than $N^2$.

\section{Generic True Covariance}
\label{sec:generic_true_cov}

In this section, we consider the more general and challenging setup where the true covariance matrix and consequently its estimate are dense, which makes the computation of VNNs heavy, or \changed{no knowledge about the sparsity of the true covariance is available}. 
While thresholding strategies can be applied to this setting, choosing an appropriate threshold is challenging; a too large threshold might lead to an S-VNN that is significantly different from the VNN operating on the true covariance due to largely truncated correlations, ultimately leading to decreases in performance, while a too small threshold might lead to an irrelevant computational advantage, thus defeating the purpose of sparsification. 
To overcome this issue and reduce computational cost in a tractable manner, we propose a stochastic sparsification framework in a form akin to dropout. Such an approach is general and does not require any structure assumption on the true covariance matrix, and it also extends the former stability analysis from small to large perturbations as stochastic sparsification may lead to large changes on the sample covariance matrix.

\begin{definition}[Stochastic sparsification]
\label{def:stochastic_sparsification}
\changed{Let $\mtM$ be a matrix with the same support as the sample covariance $\mthC$, with entries $m_{ij} = m_{ji} = 1$ with probability $p_{ij}$ and 0 otherwise (i.e., $\textnormal{Bernoulli}(p_{ij})$) and $m_{ii} = 1$. 
A sparsified covariance matrix is $\mttC = \mtM \odot \mthC$, where $\odot$ is the elementwise product.}
\end{definition}

Stochastic sparsification generates randomly sparsified matrices of the sample covariance matrix via element-wise independent sampling, and allows the sparsified matrix $\mttC$ to preserve: (i) the symmetric property in consistency with the covariance principle; and (ii) the variance of the data points on the main diagonal. 

\begin{figure}
\centering
\includegraphics[width=.5\linewidth,trim={0cm 1.2cm 0 2cm},clip]{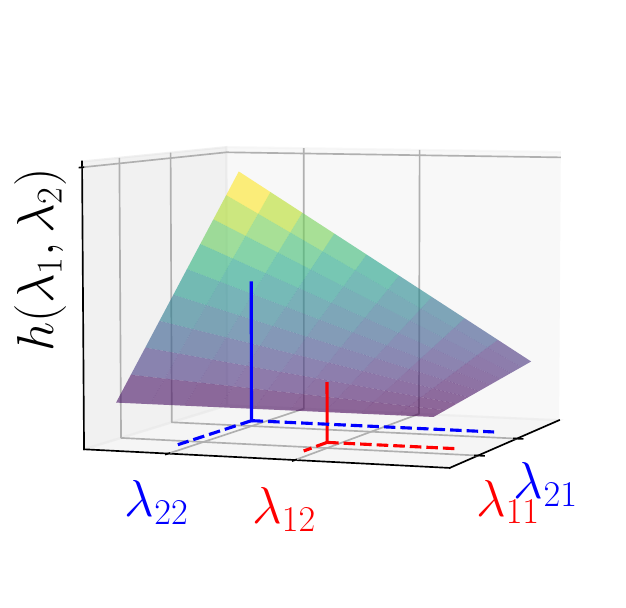}
\caption{Two-dimensional generalized frequency response $h(\vclambda)$ of stochastic sparsified covariance filter. $h(\vclambda)$ is determined by the coefficients $\{h_k\}_{k=0}^K$ and is independent of the covariance realizations. For specific realizations, $h(\vclambda)$ is instantiated on the corresponding multivariate frequencies (e.g., $\lambda_{11},\lambda_{12}$ and $\lambda_{21},\lambda_{22}$ for two different realizations $\mttC_1\mttC_2$ and 
$\mttC'_1\mttC'_2$).}
\label{fig:2d_freq_resp}
\end{figure}

\subsection{S-VNNs with Stochastic Sparsification}
\label{subsec:generic_true_cov_theory}

We now investigate the effects of stochastic sparsification on the S-VNN by formulating a stochastic perturbation problem following~\cite{gao2021stability}. We start by defining stochastic covariance filters. 
\begin{definition}[Stochastic covariance filter]
\label{def:cov_filter_random_cov}
    Given a sequence of i.i.d. realizations $\mttC_k,\dots,\mttC_1$ of randomly 
    sparsified covariance $\mttC$ [Def. \ref{def:stochastic_sparsification}], a stochastic covariance filter $\mtH(\mttC)$ performs convolution of a generic signal $\vcx$ as $\vctu = \mtH(\mttC)\vcx = \sum_{k=0}^Kh_k\mttC_k\dots\mttC_1\mttC_0\vcx$ with $\mttC_0 = \mtI$. 
\end{definition}
Stochastic filters shift the signal $\vcx$ over $K$ different random realizations $\{\mttC_k\}_{k=1}^K$ of $\mttC$ rather than shifting over a fixed $\mthC$. 
This leads to a generalized frequency interpretation.

\begin{definition}[Generalized covariance filter frequency response]
\label{def:generalized_freq_resp}
The generalized frequency response of the stochastic covariance filter is a multivariate function of the form
        $h(\vclambda) = \sum_{k=0}^Kh_k\prod_{\kappa=0}^k\lambda_\kappa$
    where $\lambda_0=1$ and $\vclambda = [\lambda_1, \dots, \lambda_K]^\Tr$ is a generic vector variable with each frequency variable $\lambda_k$ corresponding to the covariance realization $\mttC_k$ for $k=1,\ldots,K$.
\end{definition}

The derivations of such a frequency response are reported in Appendix \ref{app:generalized_freq_derivation} in the supplement. 
Intuitively, the generalized frequency response is a multivariate function in the eigenvalues of each random covariance realization $\mttC_k$ at each shift $k$, whereas for a fixed $\mthC$, the eigenvalues are the same for each shift and therefore the resulting function is a univariate polynomial.  
\Cref{fig:2d_freq_resp} illustrates an example of $K=2$. This extends the covariance filter frequency response in Sec. \ref{subsec:CF} to the stochastic setting.

Given the generalized covariance filter frequency response, we generalize the integral Lipschitz property for stability analysis.

\begin{definition}[Lipschitz gradient]
\label{def:lipschitz_gradient}
    Consider the analytic generalized frequency response $h(\vclambda)$ and two instantiations $\vclambda_1 = [\lambda_{11},\dots,\lambda_{1K}]^\Tr$ and $\vclambda_2 = [\lambda_{21},\dots,\lambda_{2K}]^\Tr$ of vector variable $\vclambda$. Consider also the auxiliary vector that concatenates the first $k$ entries of $\vclambda_2$ and the last $K-k$ entries of $\vclambda_1$, i.e., $\vclambda^{(k)} \!=\! [\lambda_{21},\dots,\lambda_{2k},\lambda_{1(k+1)},\dots,\lambda_{1K}]^\Tr$. The Lipschitz gradient of $h(\vclambda)$ between $\vclambda_1$ and $\vclambda_2$ is
    \begin{align*}
        \nabla_Lh(\vclambda_1,\vclambda_2) = \left[ \partial h(\vclambda^{(1)})/\partial\lambda_1, \dots, \partial h(\vclambda^{(K)})/\partial\lambda_K \right]^\Tr
    \end{align*}
    where $\partial h(\vclambda^{(k)})/\partial\lambda_k$ is the partial derivative w.r.t. $\lambda_k$ at $\vclambda^{(k)}$.
\end{definition}

\begin{definition}[Generalized integral Lipschitz filter]
\label{def:filter_lipschitz}
    A covariance filter is generalized integral Lipschitz if there exists a constant \changed{$P_{gi}$} s.t. 
        $\|\nabla_Lh(\vclambda_1,\vclambda_2)\| \leq P_{gi}$ and $\|\vclambda_1 \odot \nabla_Lh(\vclambda_1,\vclambda_2)\| \leq P_{gi} $ \changed{for
        $\lambda_{1i},\lambda_{2i}\in [-\lambda_{\text{min}}, \lambda_{\text{max}}],$ $i=1,\dots,K$ with $\lambda_{\text{min}}, \lambda_{\text{max}}$ large enough to include the eigenvalues of the sparsified covariance realizations.}
\end{definition}

The Lipschitz gradient characterizes the variability of $h(\vclambda)$ since, for two multivariate frequency vectors $\vclambda_1,\vclambda_2$, we have $h(\vclambda_2)-h(\vclambda_1)=\nabla_L^\Tr h(\vclambda_1,\vclambda_2)(\vclambda_2-\vclambda_1)$.
The generalized integral Lipschitz filter limits this variability to be at most linear in the multidimensional space and to decrease as the frequency $\vclambda$ is specified at large values, which extends the standard integral Lipschitz property to the multivariate frequency domain. 

These preliminaries allow us to analyze the stability of S-VNNs with stochastic covariance sparsification. 
\begin{theorem}
\label{th:stability_random}
    Consider a sparse covariance filter with stochastic sparsification $\mtH(\mttC)$ [cf. Def.~\ref{def:cov_filter_random_cov}] that is generalized integral Lipschitz with constant $P_{gi}$ [cf. Def.~\ref{def:filter_lipschitz}]. Let also $\mtH(\mtC)$ be the same filter operating on the true covariance matrix. Then, for a generic signal $\|\vcx\|\leq 1$, the expected squared difference between the two filters can be upper-bounded with probability at least $1-t^{-2\epsilon}- 2\kappa N/t$ for any $\epsilon \in (0, 1/2]$ as
    \begin{align}
    \nonumber
        \mathbb{E}[\|\mtH(\mtC)\vcx-\mtH(\mttC)\vcx\|^2 | \changed{\mthC}] \leq \underbrace{P_{gi}^2Q + \mathcal{O}((1-p_1)(1-p_2))}_\textnormal{sparsification error} + \\
        \nonumber\underbrace{\frac{P^2k_\textnormal{max}^2}{t^{1-2\epsilon}}\mathcal{O}\left(N + \frac{\|\mtC\|^2\log(Nt)}{k_\textnormal{min}^2t^{4\epsilon}} \right)}_\textnormal{covariance uncertainty}
    \end{align}
    where $Q=\sum_{i=1}^N\sum_{n=1}^N \schc_{in}^2(1-p_{in})$, the residual $\mathcal{O}((1-p_1)(1-p_2))$ for generic probababilities $p_1,p_2$ collects higher-order probability terms that are negligible compared the linear contributions in the probability value in $Q$, and $k_\textnormal{max}, k_\textnormal{min}, \kappa$ are defined in \Cref{th:cov_filter_stability}.
\end{theorem}

\Cref{th:stability_random} identifies two main factors that affect the stability of sparse covariance filters with stochastic sparsification, i.e., the covariance uncertainty and the sparsification error.

\underline{\textit{Covariance uncertainty.}} 
This term is analogous to the stability bound of VNN in \Cref{th:cov_filter_stability}. It decreases with the number of samples as $\mathcal{O}(1/{t^{1-2\epsilon}})$ such that the overall bound is dominated by the sparsification error for sufficiently large $t$ (i.e., the error introduced by sparsification prevails on the finite-sample estimation error). 

\underline{\textit{Sparsification error.}}
This term decreases as the sampling probabilities $p_{ij} \to 1$, corresponding to an improved stability but a lower sparsification. This indicates a trade-off between the perturbation effect caused by the covariance sparsification error and the computational cost saved by the stochastic sparsification. The stability constant depends on the filter Lipschitz constant \changed{$P_{gi}$}, the data dimension $N$, and the coupling between the covariance values $\hat{c}_{ij}$ and the sampling probabilities $p_{ij}$, i.e., $Q$. A larger \changed{$P_{gi}$} allows for a higher filter discriminability as the frequency response can change more quickly and identify the difference between nearby frequencies, but leads to worse stability. 
Data with high dimensionality $N$ results in a larger graph and increases the effect of sparsification on the VNN stability. 
More importantly, $Q$ represents the interplay between sample covariance values $\schc_{ij}$ and their corresponding sampling probabilities $p_{ij}$, which can be used as a design choice to develop stochastic sparsification strategies in a principled manner as we shall discuss next. 
This term allows to capture the impact of large perturbations caused by sparsification, making our analysis principally different from \cite{sihag2022covariance} which is limited to small covariance perturbations.
Finally, 
the results in Thm.~\ref{th:stability_random} and its proof generalize and differ substantially from the results in~\cite[Theorem 1]{gao2021stability} as we consider here edge-specific probabilities $p_{ij}$ rather than an identical probability $p_{ij} = p$ and we identify a connection between the S-VNN stability and the data distribution through their covariances $\schc_{ij}$, which is not present in~\cite{gao2021stability}.

\begin{figure}[t]
  \centering
\includegraphics[width=\linewidth]{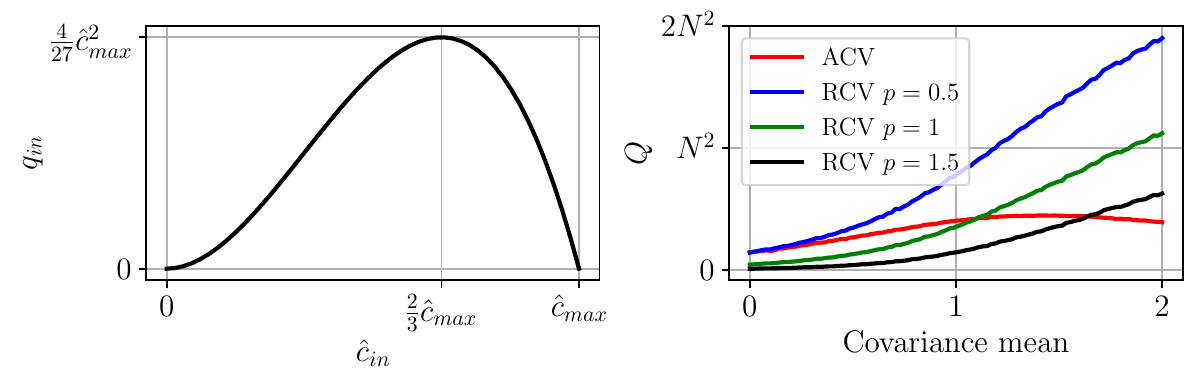}
  \caption{Insights into stability for stochastic sparsification. (Left) Stability bound elements $q_{in} = \schc_{in}^2(1-|\schc_{in}| / \schc_\text{max})$ with $Q=\sum_{i=1}^N\sum_{n=1}^N q_{in}$ for ACV [cf. \Cref{th:stability_random}]. (Right) Stability term $Q$ [cf. \Cref{th:stability_random}] for covariance values distributions with different means and stochastic sparsifications.}
  \label{fig:stab_insights}
\end{figure}


\subsection{Sparsification Probability}\label{sec:designsparse}
Following the interplay between the sample covariance values and sparsification probabilities on the S-VNN stability, we propose two stochastic sparsification strategies termed: absolute covariance values and ranked covariance values. 
\changed{Note that conditioning the expectation in Theorem 5 on $\mthC$ allows us to define $p_{ij}$ as functions of $\mthC$ without introducing additional randomness.}

\smallskip \noindent 
\textbf{Absolute covariance values (ACV).} Due to spurious correlations, the sample covariance matrix may contain some values of lower magnitude. We want to drop these small values with a high probability to save computational cost while keeping large covariance values to maintain the useful data correlations. 
Thus, we define the probability as $p_{ij} = |\schc_{ij}| / \schc_\text{max}$, where $\schc_\text{max} = \max_{i,j}|\schc_{ij}|$. \Cref{fig:stab_insights} (left) shows that the stability term $q_{in} = \schc_{in}^2(1-|\schc_{in}| / \schc_\text{max})$ with $Q=\sum_{i=1}^N\sum_{n=1}^N q_{in}$ [cf. \Cref{th:stability_random}] is small when the covariance value $\schc_{in}$ approaches zero or the maximal value $\hat{c}_{\max}$. This is because small covariance values have little impact on the S-VNN stability even if they are more likely to be dropped, and large covariance values (close to $\schc_\text{max}$) are less likely to be dropped though their removal would affect stability more. 

\smallskip \noindent 
\textbf{Ranked covariance values (RCV).} While ACV improves efficiency by preserving stronger correlations, it does not allow to control the amount of sparsification (e.g., if all covariance values are high, very few are dropped). To overcome this, we define a set of probabilities with a desired mean $p$ and assign them to covariance values based on their positions in the absolute ranking. Formally, consider an ordered set of probabilities $\mathcal{P}=\{p'_1, \dots, p'_{N'}\}$, where $p'_i \leq p'_{i+1}$, and each $p'_i$ is sampled from $\mathcal{N}(p,\sigma)$ with $\sigma = \min((1-p)/3, p/3))$ (such that the number of values not in $[0,1]$, which we clip to the interval, is negligible) and $N'$ is the number of probability values for assignment. We set $p_{ij} = p'_{k}$ where $k = |\{\scc_{lm}:|\scc_{lm}|<|\scc_{ij}|;l,m=1,\dots,N\}|$ is the position of $c_{ij}$ in the ordered ranking of absolute covariances.
Consequently, the expected percentage of dropped covariances is $1-p$, which allows controlling the sparsification level at the risk of removing useful covariances.
The value of $p$ can be treated as a hyperparameter and tuned through cross-validation or on a validation set based on the desired accuracy and stability.

\Cref{fig:stab_insights} (right) shows how the term $Q$ changes for ACV and RCV with different means of covariance value distribution.  
RCV with a smaller $p$ corresponds to a lower stability especially when the covariance values are high, because they are dropped regardless of their value. ACV, instead, maintains a consistent level of stability since it balances the covariance magnitudes and the dropping probabilities, but it does not allow to control the level of sparsification. 

\smallskip
\noindent\textbf{Computational complexity.}
The expected computational complexity of a S-VNN layer with stochastic sparsification is of the order $\mathcal{O}(\mathbb{E}[\|\mttC\|_0]KF_\text{in}F_\text{out})$, where $\mathbb{E}[\|\mttC\|_0]$ is the expected number of non-zero elements of the sparsified covariance. For ACV, we have that $\mathbb{E}[\|\mttC\|_0] = \schc_\text{mean}(\|\mthC\|_0-N)+N$, where $\schc_\text{mean} = 1/(N^2-N)\sum_{i,j=1,\dots,N,i\neq j}|\schc_{ij}|$, so the computation complexity depends on the data. 
For RCV, we have $\mathbb{E}[\|\mttC\|_0] = p(\|\mthC\|_0-N)+N$, so the computation complexity can be controlled through the desired mean $p$. 

\begin{remark} 
The stochastic sparsification can be considered as an extension of the hard thresholding in Sec.~\ref{sec:sparse_true_covariance}, which allows to control the degree of sparsification through the sampling probabilities $\{p_{ij}\}$. It reduces to the hard thresholding by setting probabilities $p_{ij} = 1$ if $|c_{ij}|>\tau/\sqrt{t}$, and $0$ otherwise.
\end{remark}
\begin{remark}
Both hard/soft thresholding and stochastic sparsification provide symmetric covariance estimates, but do not necessarily preserve positive semidefiniteness (PSD). While this is common~\cite{bickel2008covariance,deshp2016sparse,li2021braingnn,liao2022har}, if PSD estimates are of interest one can perform diagonal loading $\mthC + \delta\mtI$ for a $\delta\geq 0$ that ensures all eigenvalues are non-negative.
Moreover, \cite{bickel2008covariance} provides a
sufficient condition for PSD under hard-thresholding. Specifically, 
consider a hard-thresholded sample covariance matrix $\mtbC$ as per \Cref{def:hard_thr} and the sample covariance matrix $\mthC$.
The sparsified $\mtbC$ is PSD if $\|\mtbC - \mthC\| \leq \epsilon$ and $\lambda_\textnormal{min}(\mthC) > \epsilon$ for an $\epsilon > 0$, where $\lambda_\textnormal{min}(\cdot)$ computes the smallest eigenvalue. That is, if the true covariance matrix is sparse, then the term $\|\mtbC - \mthC\|$ decreases as the number of samples $t$ increases and the sparse estimate is more likely to be PSD.
\end{remark}

\section{Numerical Results}
\label{sec:numerical_results}

We evaluate the proposed S-VNN with different sparsification strategies through experiments on both real and synthetic datasets, targeting the following objectives: (\textbf{O1}) validate the sparsification strategies and theoretical results when the true covariance is sparse; (\textbf{O2}) validate the sparsification strategies and theoretical results when the true covariance matrix is dense; 
(\textbf{O3}) evaluate the performance and computation efficiency of the S-VNN over several baselines on real-world data.

\subsection{Datasets}

We consider one controlled scenario and four real datasets.

\noindent 
\textbf{Synthetic scenario.} 
We evaluate S-VNNs on a synthetic regression task, analogously to~\cite{sihag2022covariance}.
We generate synthetic data with a controlled covariance by sampling data points $\vcx_i \sim \mathcal{N}(\mathbf{0}, \mtC)$. 
Then, we create regression targets as $y_i = \vcw^\Tr\vcx_i + u$, where $\vcw$ is a vector with elements $w_j \sim \text{Uniform}(0,1)$ and $u \sim \mathcal{N}(0,3)$ is a noise term. We generate 1000 samples of size $N=100$ and divide the data into train/validation/test splits of size 80\%/10\%/10\%.
We consider three different conditions: we use a sparse $\mtC$ (SparseCov), a dense $\mtC$ with large values (LargeCov) and, for the case of a dense $\mtC$ with small values (SmallCov), we use the synthetic linear regression dataset in~\cite[Appendix E.4]{sihag2022covariance} with tail=0.2. We show the covariance distributions of LargeCov and SmallCov in \Cref{fig:distr_cov}.

\begin{table}[t]
\centering
\footnotesize
\caption{Details on datasets.}
\resizebox{\linewidth}{!}{
\begin{tabular}{c|cccc}
\toprule
\textbf{Dataset} & \textbf{Nodes} & \textbf{Node} & \textbf{Classes} & \textbf{Samples} \\
 & & \textbf{features} & & \textbf{(train/valid/test)} \\
\midrule
MHEALTH & 23 & 128 & 12 & 5357 (3239/1028/1090) \\
Realdisp & 117 & 128 & 10 & 5660 (2683/1258/1719) \\
ADNI1 & 68 & 1 & NA & 801 (481/160/160) \\
ADNI2 & 68 & 1 & NA & 1142 (685/228/229) \\
\bottomrule
\end{tabular}}
\label{tab:datasets}
\end{table}

\begin{figure*}[t]
\centering
\includegraphics[width=\linewidth]{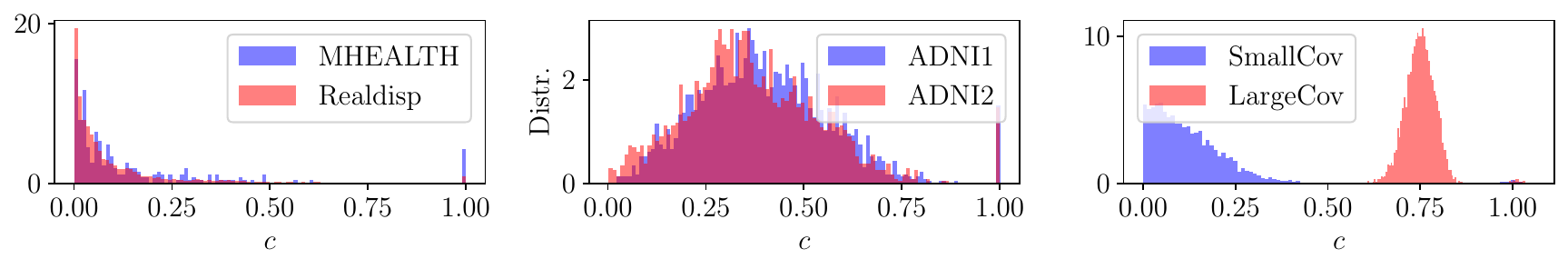}
\caption{Normalized absolute covariance values distributions for real and synthetic datasets. }
\label{fig:distr_cov}
\end{figure*}

\noindent \textbf{Real datasets.}
We also consider four real-world datasets of brain data and human action recordings for S-VNN stability and performance evaluation.
\Cref{tab:datasets} summarizes the main information about the datasets and \Cref{fig:distr_cov} shows their empirical normalized covariance values distribution.
\begin{itemize}

\item \textit{MHEALTH}~\cite{banos2014mhealth} contains measurements of wearable devices placed in the chest, right wrist and left ankle of 10 subjects performing 12 different actions. \changed{Our data preprocessing follows~\cite{tello2023good}: we use recordings from different subjects for train, validation and test (specifically, subjects 6, 10 for validation, 2, 9 for test and the rest for training); we do not use the ECG measurements placed in the chest; we use a sampling rate of 50 Hz.} The goal is to classify the action performed by the subject. 
\item \textit{Realdisp}~\cite{misc_realdisp_activity_recognition_dataset_305} contains data of 17 subjects performing 33 actions. We use subjects 10, 11, 12 for validation; subjects 13, 14, 15, 16, 17 for test and the remaining for training. We use the 10 most common activity labels: walking, jogging, running, cycling, elliptic bike, trunk twist (arms outstretched), rowing, knees (alternatively) bend forward, waist bends forward, trunk twist (elbows bended). We segment the data in MHEALTH and Realdisp creating sliding windows of size 128 (2.56 seconds) for measurements relative to the same activity with 50\% overlap. 

\item \textit{ADNI1} and \textit{ADNI2}~\cite{jack2008alzheimer} correspond to two different phases of the ADNI project containing cortical thickness measures of patients affected by Alzheimer's disease and healthy control patients. The phases differ in the collection protocols and historical periods. 
We download the cortical thickness measures for both datasets (i.e., the outputs of the FreeSurfer~\cite{fischl2012freesurfer} processing of the brain MRI recordings) from \url{https://ida.loni.usc.edu/}. For each patient, we consider the first scan available, generally corresponding to the screening visit. The task is a regression problem that aims to predict the chronological age of the patient at the time of the scan from its cortical thickness measures. This is a task of high interest to identify neurodegenerative diseases from early brain aging, since cortical thickness generally decreases with age. Correlations among brain regions play a relevant role in characterizing cortical thickness patterns~\cite{sihag2024explainable,yin2023anatomically,bashyam2020mri,couvy2020ensemble}, making this task relevant for S-VNNs.

\end{itemize}

\begin{figure}[t]
  \centering
  \subfloat
{\includegraphics[width=.7\linewidth]{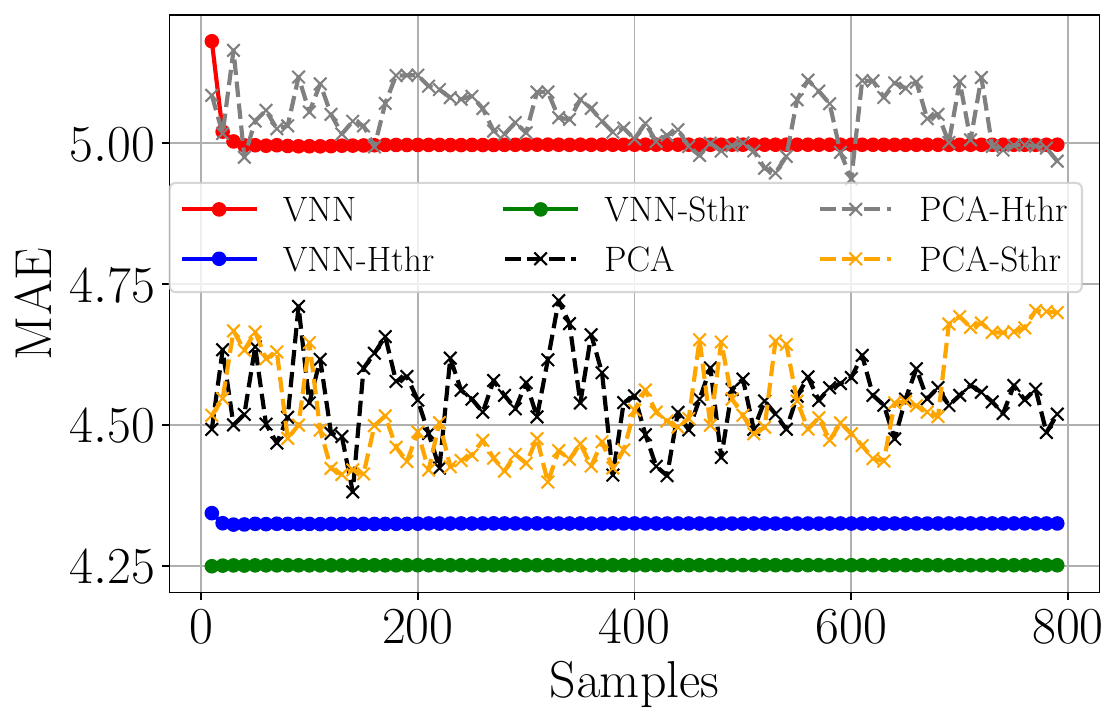}
} \hfill
    \subfloat{    
    \includegraphics[width=.7\linewidth]{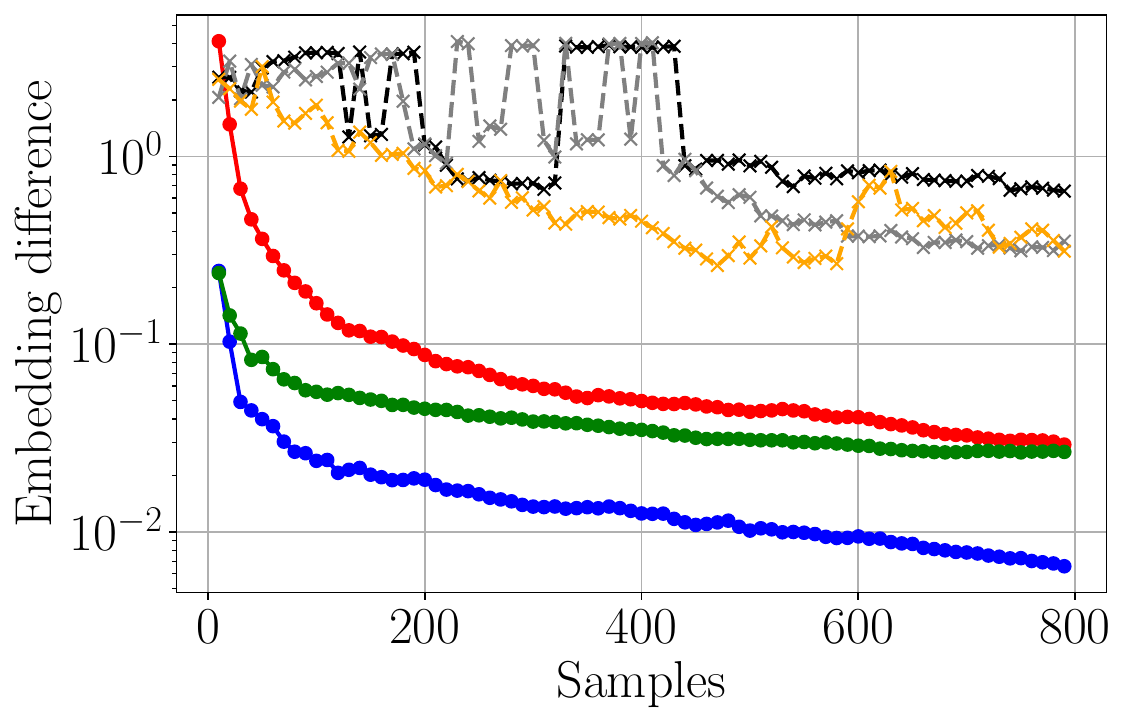}}
  \caption{Stability of hard and soft thresholded S-VNNs, nominal VNN and PCA-SVM with dense and sparse covariance on regression task for a synthetic dataset with sparse true covariance. (Above) Regression performance in terms of Mean Absolute Error. (Below) Embedding difference (i.e., $\| \fnPhi(\vcx, \mtbC, \mathcal{H}) - \fnPhi(\vcx, \mtC, \mathcal{H}) \|$) between S-VNN/VNN/PCA with true and estimated covariance. Standard deviations in our results are of order at most $10^{-2}$ for VNNs and $10^{-1}$ for PCA.}
  \label{fig:true_sparse}
\end{figure}

\begin{figure*}[t]
  \centering
  \subfloat{
  \includegraphics[width=1\linewidth]{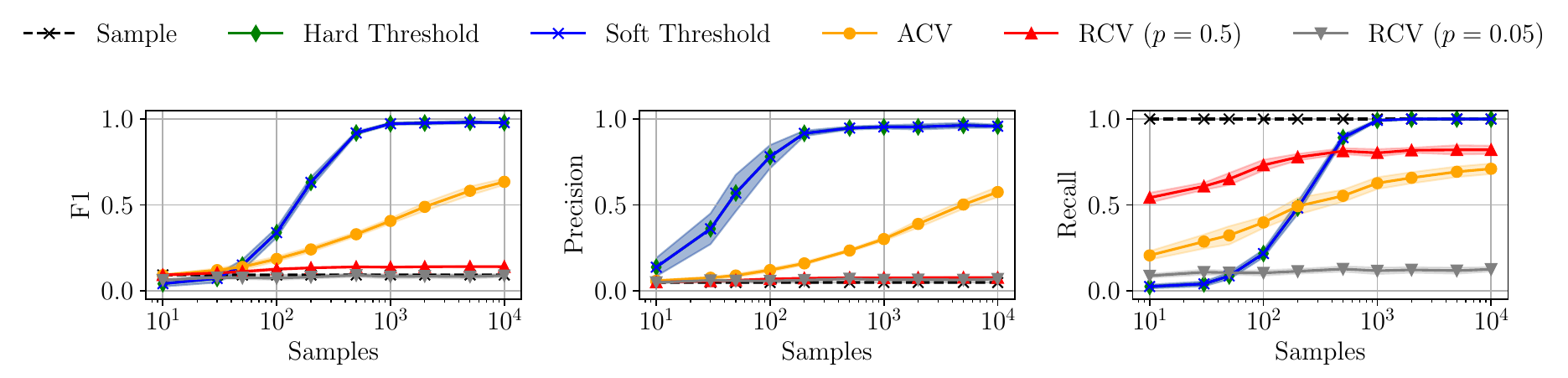}
}\hfill
\subfloat{
\includegraphics[width=.94\linewidth]{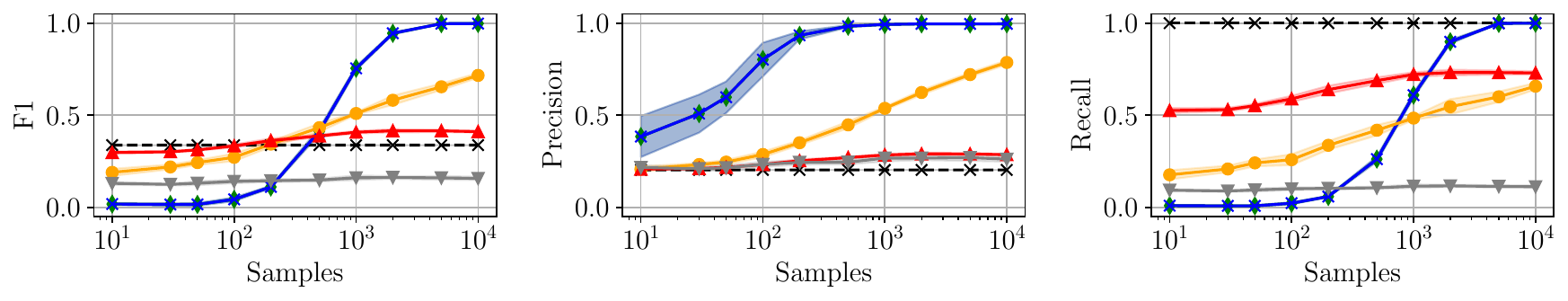}
}
\caption{\changed{F1 score, precision and recall in identifying non-zero entries of a sparse true covariance with different estimators. We set $N=100$ and vary the number of samples $T$. True sparsity is $5\%$ (first row) and $20\%$ (second row).} }

\label{fig:recovery_05}
\end{figure*}


\changed{\subsection{Validation of Thresholding Methods}
\label{sec:thr_validation}
\noindent 
%
\textbf{Support recover on synthetic data -- (\textbf{O1}).}
We first validate the deterministic and stochastic thresholding methods by assessing their capability to identify spurious correlations in a synthetic setting. Specifically, we generate a true covariance matrix $\mtC$ with either $20\%$ or $5\%$ non-zero entries, we draw $T$ samples from a Gaussian distribution with covariance $\mtC$ and we use them to estimate the covariance with various estimators. Fig.~\ref{fig:recovery_05} shows that the deterministic thresholding methods achieve perfect recovery as the number of samples increases. The ACV generally improves over the sample estimator, but does not achieve perfect recovery due to its stochasticity. The RCV, finally, is less accurate due to its stochastic nature---it is designed to achieve a desired level of sparsity stochastically, which might not correspond to effective support recovery in all cases. 
This analysis suggests that while some sparsification methods improve downstream performance of SVNNs due to their capability of recovering the true covariance support (e.g., thresholding), others act as regularizers/dropout techniques due to their stochasticity (e.g., ACV, RCV) and help learn more robust models.  
}

\subsection{Stability of S-VNN}
\label{sec:stability_exp}

\changed{We now proceed to validate the stability of S-VNNs.}

\smallskip \noindent 
\textbf{Sparse true covariance on synthetic data, \changed{changing $T$} -- (\textbf{O1}).}
\label{subsec:exp_sparse_true} We evaluate the proposed hard thresholding (Hthr) and soft thresholding (Sthr) using the SparseCov dataset. We train a VNN with the true covariance and test it with different covariance estimates under varying number of samples. We consider the nominal VNN with the original sample covariance and the PCA-SVM as baselines for comparison. For PCA-SVM,   
we transform data with PCA and we then apply an RBF-SVM for regression.
\Cref{fig:true_sparse} reports the Mean Absolute Error (MAE) on the downstream task and the embedding difference between the models operating on the true and sample covariances, w.r.t. the number of samples for covariance estimation. 
Overall, hard and soft thresholding provide better covariance estimates, ultimately, improving the performance (i.e., lower MAE) and increasing the stability (i.e., lower embedding difference) compared to the nominal VNN. S-VNN and VNN are more stable than PCA-SVM, because their performance is less affected by covariance perturbations as analyzed in Sec. \ref{subsec:CF}.  
These results support our theoretical discussion in \Cref{sec:sparse_true_covariance}, i.e., thresholding allows VNNs to maintain or even improve stability in sparse settings.

\begin{figure}[t]
  \centering
  \includegraphics[width=1\linewidth]{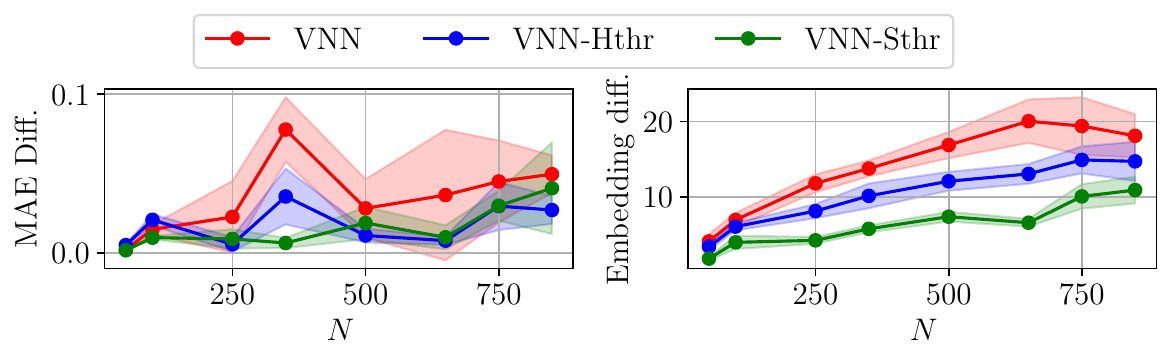}
\caption{\changed{Stability for increasing $N$ and fixed $T=250$.} 
}
\label{fig:stab_changeN}
\end{figure}

\smallskip \noindent 
\changed{\textbf{Sparse true covariance on synthetic data, changing $N$ -- (\textbf{O1}).}
We also evaluate the stability of VNNs and S-VNNs for synthetic datasets with increasing $N$. For each value of $N$, we train them with the clean covariance and test them on the sample covariance from $T=250$ samples with and without thresholding. We report the change in downstream performance (MAE Diff.) and the difference in the embeddings in Fig.~\ref{fig:stab_changeN}. For both metrics, SVNNs achieve better stability than VNNs, corroborating their superior stability under different $T$ and $N$.
}

\smallskip
\noindent 
\textbf{Generic true covariance on synthetic data -- (\textbf{O2}).}
We assess the stochastic sparsification using LargeCov and SmallCov.  We train a VNN with the true covariance matrix and test it with
stochastically sparsified sample covariance matrices. %
\Cref{fig:stab_syth} shows the performance of the models operating over the true covariance and the sample covariance with the proposed ACV and RCV of different sparsification means $p$. 
For LargeCov, we see that a high sparsification leads to greater degradation (RCV for small $p$) as large covariance values are dropped which hinders performance. For SmallCov, instead, all sparsification approaches affect MAE only lightly as removing small covariance values does not lead to significant performance changes. ACV maintains stable performance for both LargeCov and SmallCov because it sparsifies little when the covariance values are large, allowing, however, for relatively small computational improvements. RCV, instead, allows for consistent sparsification levels and faster computation at the cost of more degradation when covariance values are large. 
These results corroborate our stability analysis in Theorem \ref{th:stability_random} and theoretical observations in \Cref{subsec:generic_true_cov_theory}. 

\smallskip
\noindent \textbf{Stochastic sparsification on real datasets -- (\textbf{O2}).} 
\Cref{fig:stab_prob_real} shows the stability of S-VNNs under stochastic sparsification with different strategies on the 4 real-world datasets. 
On Realdisp, which contains mostly small covariance values, the effect of sparsification is minimal, as the accuracy of S-VNNs remains consistent for most values of $p$. This indicates that most of the small correlations present in this dataset are not relevant for the downstream task. 
On ADNI1 and ADNI2, the covariance sparsification leads to a significant impact on the regression error for $p < 0.5$, as these datasets contain large covariance values (cf. Fig.~\ref{fig:distr_cov}) whose removal affects the downstream performance. 
On MHEALTH, instead, the accuracy of S-VNN drops significantly for $p<0.5$. This indicates that, despite the covariance values being small, they still carry relevant information for the downstream task. Moreover, the classification task for MHEALTH contains 12 classes, making it more challenging than the age regression task on brain datasets, which may lead to lower model stability.

\begin{figure}[t]
  \centering
    \subfloat{
    \includegraphics[width=.49\linewidth]{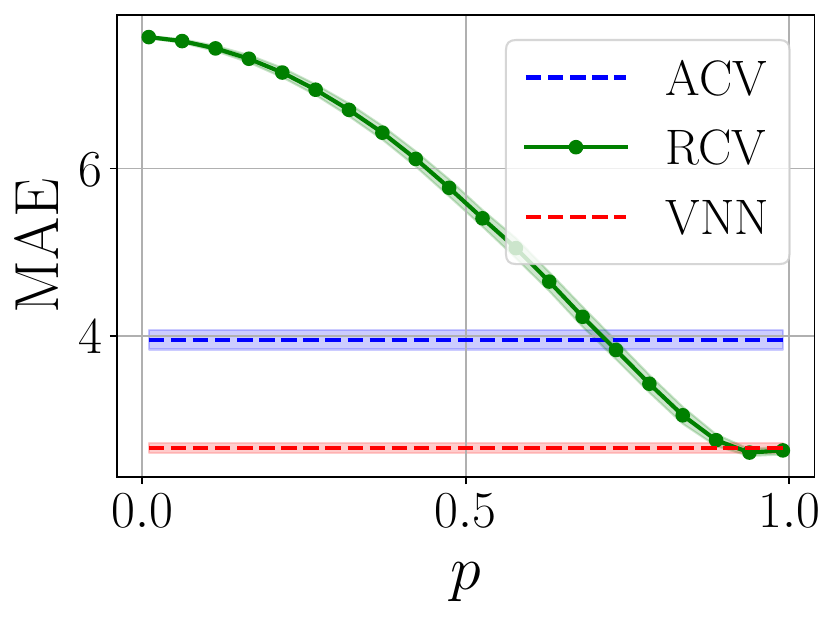}}
    \hfill
    \subfloat{
    \includegraphics[width=.47\linewidth]{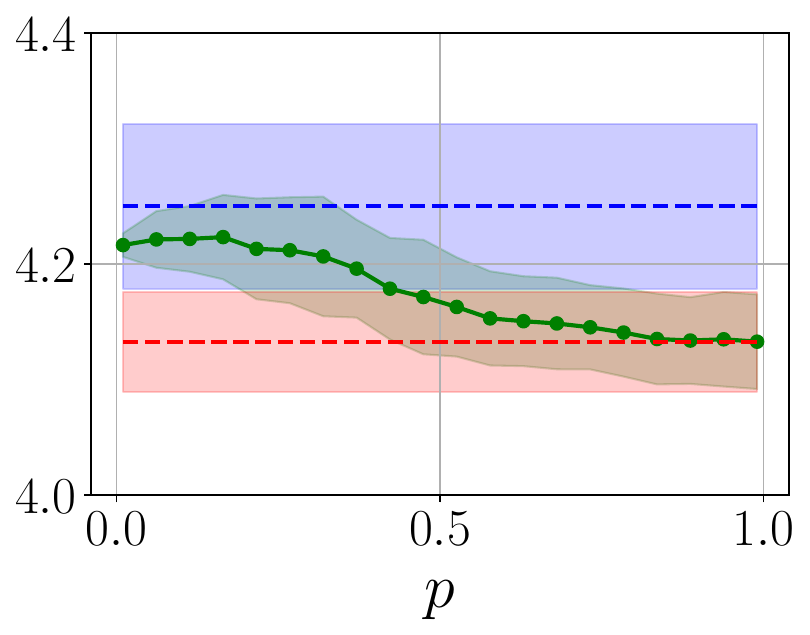}}
  \caption{Regression performance in terms of MAE for S-VNNs with different stochastic sparsification techniques. (Left) LargeCov and (Right) SmallCov.}
 \label{fig:stab_syth}
\end{figure}

\begin{figure*}[t]
  \centering
  \includegraphics[width=1\linewidth]{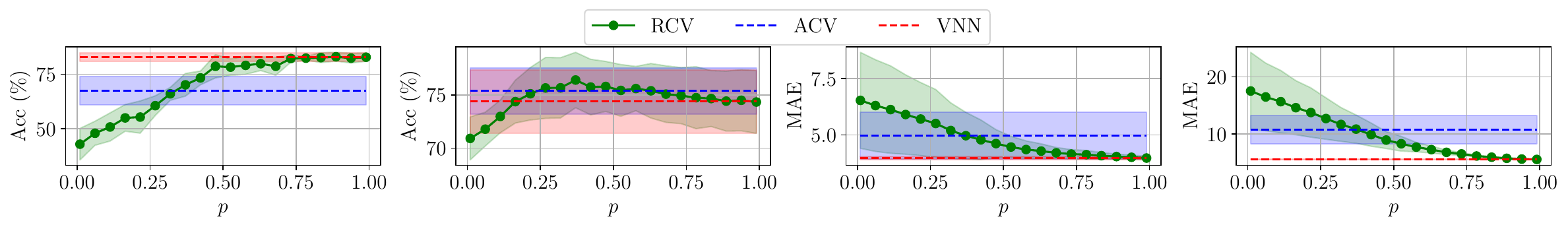}
\caption{Stability of S-VNNs (in terms of average accuracy and standard deviation) on MHEALTH, Realdisp, ADNI1 and ADNI2 (from left to right) with different stochastic covariance sparsification techniques. 
}
\label{fig:stab_prob_real}
\end{figure*}

\begin{figure*}[t]
  \centering
  \includegraphics[width=\linewidth]{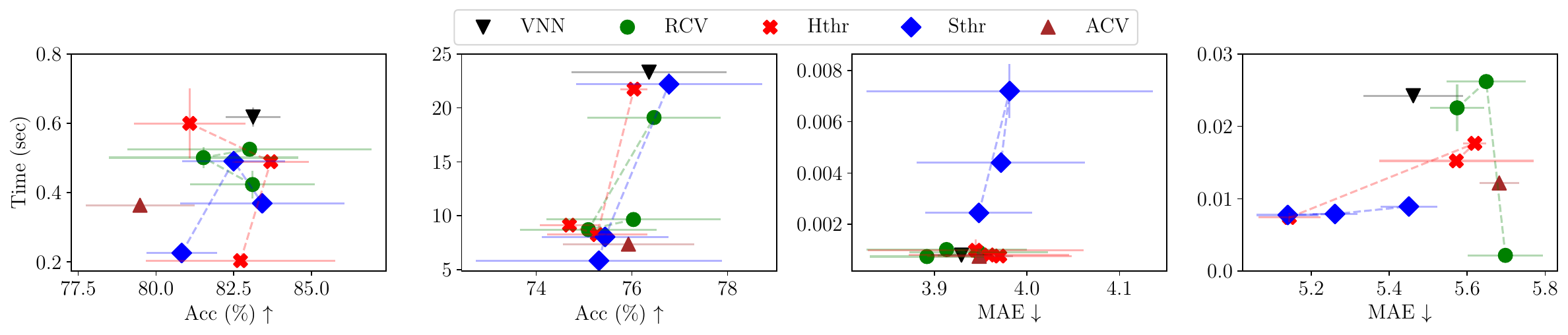}
  \caption{MAE/Accuracy (\%) and time for a forward pass on the test set for sparse and dense VNNs. From left to right, the results are for MHEALTH Realdisp, ADNI1 and ADNI2. RCV results are for $p=0.75,0.5,0.25$ from top to bottom, and for Hthr and Sthr we report results for 3 different thresholds to achieve analogous sparsification. }
 \label{fig:realworld}
\end{figure*}

\subsection{Performance of S-VNN}
\label{sec:real_exp}

\noindent 
\textbf{Experimental setup -- (\textbf{O3}).}
We compare S-VNNs with different sparsification strategies (i.e., hard thresholding, soft thresholding, RCV, ACV) against four categories of baselines: (i) the nominal VNN~\cite{sihag2022covariance} which does not enforce sparsity; (ii) an MLP on the concatenation of raw node features that does not account for covariance information; 
(iii) an MLP on the concatenation of node features preprocessed by PCA, sparse PCA~\cite{zou2006sparse} and kernel PCA~\cite{scholkopf1997kernel}, which consider covariance information via PCA preprocessing; 
(iv) a GNN operating on a sparse graph obtained via neighbor sampling~\cite{hamilton2018inductive}, where sparsification is not aware of covariance values but random.
Our objective is to highlight the impact of sparsification, which reduces spurious correlations and improves computational efficiency of VNNs. 
For Hthr and Sthr, we compare 3 different thresholds that achieve a level of sparsification comparable to RCV with $p=0.75,0.5,0.25$ to compare their performance for analogous computational gains. 
\changed{To strengthen the experimental comparison, we include two variants of S-VNNs with Ledoit-Wolfe (LW) covariance estimator~\cite{ledoit2003honey} and Graphical Lasso (GL)~\cite{friedman2007lasso}. While we do not discuss these variants in theory, our framework is flexible to include other regularized estimators in practice.}

\noindent 
\textbf{Results -- (\textbf{O3}).} From \Cref{fig:realworld}, we see that the S-VNNs decrease substantially the computation time and improve the downstream performance compared with the VNN. 
The time improvements are evident on ADNI2 and especially the larger MHEALTH and Realdisp for most S-VNNs. On ADNI1, given its small size, the computation time is overall low for all models (if compared, for example, with ADNI2), and the impact of sparsification is less significant. Moreover, the larger computational time of Sthr compared to other models is due to its best hyperparameter configuration, which contains S-VNNs with 2 or 3 layers and order $K=3$, as opposed to the 1 layer and $K=1$ of the best configurations of other models. The Sthr with the same parameters of VNN has a computation time of $0.68\pm0.03$ (faster than VNN).

In terms of downstream performance, the best MAE and accuracy are always obtained by S-VNNs, indicating that dense covariances contain spurious correlations that are detrimental for performance, \changed{or that the regularization introduced by stochastic sparsification helps learn a more robust and effective model}. This is particularly evident on ADNI2, where several sparse variants outperform VNN.
For S-VNNs with both thresholding and RCV, changing the value of the threshold or the desired mean $p$ does not lead to large changes in accuracy, meaning that the majority of covariance values are not relevant for performance, while it affects the time efficiency of the model by increasing sparsity. This indicates that relatively higher sparsification, i.e., a larger threshold or a smaller $p$, may be a good choice in these cases. 

\changed{The LW variant achieves slightly worse performance than other S-VNN methods, due to its lack of effective sparsification. The GL variant, instead, achieves the best results on MHEALTH and Realdisp, and competitive performance on ADNI1 and ADNI2, suggesting that it effectively sparsifies the sample covariance. However, solving the graphical lasso problem incurs additional computational cost that makes this variant from 1 to 3 order of magnitude slower than other S-VNNs, making it less suitable when efficiency is paramount.}

\noindent \textbf{Comparison with MLP, PCA and neighbor sampling -- (\textbf{O3}).} 

\Cref{tab:realworld_combined} reports the results for S-VNNs, VNN (as in \Cref{fig:realworld}) and additional baselines. 
On all datasets, the best downstream performance is attained by S-VNNs, indicating the importance of leveraging covariance information in a sparse manner.
The MLP on the raw node features performs significantly worse than most other models on ADNI1, ADNI2 and Realdisp, highlighting the importance and effectiveness of considering covariance information as inductive bias. It is, however, faster than the other models as it does not involve any feature preprocessing.
Preprocessing the data with PCA and variants generally improves performance, as the relevant information is effectively summarized, but it comes at the cost of higher computational time due to the expensive covariance eigendecomposition for standard PCA and even more complex optimization problems for sparse and kernel PCA. On MHEALTH and Realdisp, kernel PCA causes an out-of-memory error due to its dependence on the number of samples which is large for these datasets. 
Finally, the neighborhood sampling approach for sparsification leads to worse downstream performance than nominal VNNs, indicating that sparsification techniques that do not consider covariance values might lose relevant information.

\begin{remark}
While the non-linear PCA and (sparse) VNNs may, at first sight, seem connected, they are fundamentally different since kernel PCA computes the principal directions of the data projected into a reproducing kernel Hilbert space, whereas VNNs manipulate the eigenvectors of the covariance matrix of the samples in the original space. Therefore, we merely use kernel PCA here as a baseline for additional comparison. 
\end{remark}

\subsection{Experimental details}
\label{sec:add_exp_details}
\noindent \textbf{Hyperparameters.}
On the synthetic datasets, we use VNNs with 2 layers, feature size 32 for LargeCov and SmallCov and 13 for SparseCov, we train for 50 epochs using Adam optimizer with learning rate 0.015 and weight decay $10^{-5}$. 
On real datasets, we perform a grid search for all models (VNN, S-VNNs and baselines) on all datasets among the following values: for VNN and S-VNNs, embedding size $F=32$, number of layers $L\in\{1,2,3\}$, dropout rate $dr \in\{0.0,0.3\}$, VNN order $K\in\{1,2,3\}$. We use the same parameter grid for the GCN except for $K$. We average the node embeddings and feed them to a 2-layer MLP as readout for the downstream task. We use batch normalization, Adam optimizer with learning rate 0.01 and weight decay $10^{-5}$ and we train for 1000 epochs with a patience of 100. 
For the MLP baseline (standalone or after PCA preprocessing), we search among three configurations: $L=2$ layers of size $32$, $L=2$ layers of size $256,128$, $L=3$ layers of size $256,128,128$. For the various PCA, we select the number of components $k\in\{10, 20, 50\}$.
\Cref{tab:params} reports the best hyperparameters after the search (for S-VNN, we report the configuration of the best-performing variant).
We do not use any diagonal loading of the sparsified covariance.

\noindent \textbf{Implementation details.}
We repeat all experiments 5 times and, for stochastic experiments, we sample 10 different sparsified covariance matrixes. We report average results and standard deviations. Experiments are run on a AMD EPYC 7452 32-Core Processor with 10 GB of RAM, on a 13th Gen Intel Core i7-1365U with 16 GB of RAM and on an NVIDIA A40 GPU. We implement models using Pytorch and for neighborhood sampling we use Pytorch Geometric~\cite{fey2019fast}. 
The sample covariance matrix $\mathbf{\hat{C}}$ is computed once over the complete training set and fixed during training, i.e., it is not affected by batching. For stability experiments, we use different sparsified covariances during test, whereas for the other experiments we keep the same matrix estimated on the training set.
\changed{For sparse PCA, we use the implementation in the scikit-learn library~\cite{scikit-learn}, which implements the sparse PCA algorithm in~\cite{mairal2009online}---a PCA problem (dictionary learning) with an $\ell_1$-penalty on the components.  }

\begin{table*}[t]
\centering
\footnotesize
\caption{Accuracy (\%) and time (sec) for a forward pass on the test set on real datasets. Time results for ADNI1 and ADNI2 are $\cdot10^{-3}$. \first{Best} and \second{second best} results are highlighted.
}
\begin{tabular}{c|l|cc|cc|cccc}
\toprule
& & \multicolumn{2}{c|}{\textbf{Accuracy (\%) $\uparrow$}} & \multicolumn{2}{c|}{\textbf{MAE $\downarrow$}} & \multicolumn{4}{c}{\textbf{Time (sec) $\downarrow$}} \\
\midrule
& & \textbf{MHEALTH} & \textbf{Realdisp} & \textbf{ADNI1} & \textbf{ADNI2} & \textbf{MHEALTH} & \textbf{Realdisp} & \textbf{ADNI1} & \textbf{ADNI2} \\
\midrule
\multirow{5}{*}{\rotatebox[origin=c]{90}{\textbf{Baselines}}}& 
MLP & 83.0$\pm$1.8 & 69.9$\pm$4.3 & 8.29$\pm$0.46 & 8.49$\pm$0.11 & 0.02$\pm$0.00 & 0.15$\pm$0.00 & 0.55$\pm$0.43 & 0.36$\pm$0.19 \\ 
& PCA + MLP & 82.2$\pm$2.0 & 72.3$\pm$1.7 & 4.84$\pm$0.19 & 5.41$\pm$0.23 & 0.26$\pm$0.01 & 3.00$\pm$0.02 & 2.91$\pm$0.54 & 4.04$\pm$1.07 \\ 
& Kernel PCA + MLP & OOM & OOM & 4.97$\pm$0.28 & 5.57$\pm$0.14 & NA & NA & 20.4$\pm$6.1 & 41.8$\pm$10.9 \\ 
& Sparse PCA + MLP & 83.2$\pm$3.2 & 68.6$\pm$4.3 & 4.76$\pm$0.17 & 5.61$\pm$0.06 & 681$\pm$99 & 419$\pm$5 & (6.0$\pm$0.5)$\cdot10^4$ & (7.9$\pm$0.3)$\cdot10^4$ \\ 
& NeighSampl~\cite{hamilton2018inductive} & 75.0$\pm$1.2 & 72.2$\pm$1.2 & 4.03$\pm$0.13 & 5.89$\pm$0.17 & 0.75$\pm$0.01 & 7.79$\pm$0.01 & 30.0$\pm$8.0 & 63.6$\pm$5.7 \\ 
& VNN~\cite{sihag2022covariance} & 83.1$\pm$0.9 & 76.4$\pm$1.6 & 3.93$\pm$0.03 & 5.46$\pm$0.13 & 0.62$\pm$0.03 & 23.3$\pm$0.3 & 0.80$\pm$0.01 & 24.2$\pm$0.4 \\ 
\midrule 
\multirow{12}{*}{\rotatebox[origin=c]{90}{\textbf{Ours}}}
& Hard-thr ($p=0.75$) & 81.1$\pm$1.8 & 76.0$\pm$0.3 & 3.97$\pm$0.08 & 5.57$\pm$0.20 & 0.60$\pm$0.10 & 21.7$\pm$0.1 & 0.76$\pm$0.03 & 15.2$\pm$0.4 \\ 
& Hard-thr ($p=0.5$) & \second{83.7$\pm$1.2} & 75.3$\pm$1.0 & 3.96$\pm$0.09 & 5.62$\pm$0.03 & 0.49$\pm$0.01 & 8.25$\pm$0.16 & 0.80$\pm$0.04 & 17.6$\pm$0.3 \\ 
& Hard-thr ($p=0.25$) & 82.7$\pm$3.0 & 74.7$\pm$0.6 & 3.94$\pm$0.12 & \textbf{5.14$\pm$0.08} & 0.20$\pm$0.01 & 9.13$\pm$0.15 & 0.98$\pm$0.42 & 7.41$\pm$0.64 \\ 
& Soft-thr ($p=0.75$)& 83.4$\pm$2.6 & \second{76.8$\pm$1.9} & 3.97$\pm$0.09 & 5.45$\pm$0.07 & 0.37$\pm$0.01 & 22.2$\pm$1.0 & 4.40$\pm$0.22 & 8.93$\pm$0.07 \\ 
& Soft-thr ($p=0.5$) & 82.5$\pm$1.7 & 75.4$\pm$1.3 & 3.98$\pm$0.15 & \second{5.26$\pm$0.06} & 0.49$\pm$0.01 & 8.04$\pm$0.17 & 7.20$\pm$1.06 & 7.89$\pm$0.19 \\ 
& Soft-thr ($p=0.25$) & 80.8$\pm$1.1 & 75.3$\pm$2.6 & 3.95$\pm$0.06 & \textbf{5.14$\pm$0.08} & 0.23$\pm$0.01 & 5.84$\pm$0.08 & 2.44$\pm$0.30 & 7.76$\pm$0.38 \\ 
& RCV ($p=0.75$) & 83.0$\pm$3.9 & 76.5$\pm$1.4 & \second{3.91$\pm$0.09} & 5.57$\pm$0.07 & 0.53$\pm$0.01 & 19.1$\pm$0.3 & 1.01$\pm$0.20 & 22.5$\pm$3.2 \\ 
& RCV ($p=0.5$) & 81.5$\pm$3.0 & 75.1$\pm$1.4 & 3.95$\pm$0.08 & 5.65$\pm$0.10 & 0.50$\pm$0.03 & 8.71$\pm$0.53 & 0.91$\pm$0.09 & 26.2$\pm$0.1 \\ 
& RCV ($p=0.25$) & 83.1$\pm$3.9 & 76.0$\pm$1.8 & \textbf{3.89$\pm$0.06} & 5.70$\pm$0.10 & 0.42$\pm$0.04 & 9.65$\pm$0.33 & 0.74$\pm$0.08 & 2.18$\pm$0.12 \\ 
& ACV & 79.5$\pm$1.7 & 75.9$\pm$2.2 & 3.95$\pm$0.04 & 5.68$\pm$0.05 & 0.36$\pm$0.01 & 7.37$\pm$0.04 & 0.75$\pm$0.05 & 12.2$\pm$0.4 \\ 
& \changed{LW} & 82.7$\pm$4.6 & 75.0$\pm$2.7 & 3.93$\pm$0.06 & 5.72$\pm$0.05 & 0.20$\pm$0.02 & 16.6$\pm$0.7 & 2.01$\pm$0.02 & 2.42$\pm$0.05 \\
& \changed{GL} & \first{85.3$\pm$3.5} & \first{76.9$\pm$3.8} & 3.92$\pm$0.06 & 5.69$\pm$0.07 & 4.20$\pm$0.08 & 45.8$\pm$0.7 & 2604$\pm$10 & 2907$\pm$10 \\
\bottomrule
\end{tabular}
\label{tab:realworld_combined}
\end{table*}


\begin{table}[t]
\centering
\footnotesize
\caption{Selected parameters on real datasets.}
\begin{tabular}{c|l|cccc}
\toprule
& & \textbf{MHEALTH} & \textbf{Realdisp} & \textbf{ADNI1} & \textbf{ADNI2} \\
\midrule
\multirow{3}{*}{\textbf{VNN}} 
& $L$  & 3 & 2 & 1 & 3 \\
& $dr$ & 0.3 & 0.3 & 0.0 & 0.0 \\
& $K$ & 3 & 2 & 1 & 2 \\
\midrule
\multirow{3}{*}{\textbf{S-VNN}} 
& $L$  & 3 & 2 & 1 & 2 \\
& $dr$ & 0.3 & 0.3 & 0.0 & 0.0 \\
& $K$ & 3 & 3 & 1 & 2 \\
\midrule
\multirow{2}{*}{\textbf{NeighSampl}} 
& $L$  & 1 & 2 & 2 & 3 \\
& $dr$ & 0.3 & 0.3 & 0.0 & 0.0 \\
\midrule
\multirow{1}{*}{\textbf{MLP}} 
& $L$  & 3 & 2 & 3 & 2 \\
\midrule
\multirow{2}{*}{\textbf{PCA+MLP}} 
& $L$  & 2 & 2 & 3 & 3 \\
& $k$  & 23 & 50 & 10 & 10 \\
\midrule
\multirow{2}{*}{\textbf{SPCA+MLP}} 
& $L$  & 2 & 2 & 3 & 3 \\
& $k$  & 20 & 20 & 10 & 10 \\
\midrule
\multirow{2}{*}{\textbf{KPCA+MLP}} 
& $L$  & NA & NA & 3 & 3 \\
& $k$  & NA & NA & 10 & 10 \\
\bottomrule
\end{tabular}
\label{tab:params}
\end{table}

\section{Related Works}

Our contribution relates to existing literature on sparse covariance estimation for PCA, stability analysis of graph and covariance neural networks, and on the sparsification strategies used for GNNs in general. In the sequel, we review these three streams and position our contribution accordingly.

\noindent
\underline{\textit{Sparse PCA.}} The instability of the finite-sample covariance matrix is widely studied for PCA~\cite{Jolliffe2002pca,paul2007asymptotics,baik2005phase} and various regularized estimators have been proposed including covariance shrinkage~\cite{ledoit2003honey,Ledoit_2012}, lasso penalties \cite{bien2011sparse,lui2014sparseeigenvalue}, and thresholding \cite{bickel2008covariance,deshp2016sparse}. We leverage here hard and soft thresholding given their benefits in PCA in low-data sparse-covariance settings and study their impact on VNN stability, and we extend the analysis to the generic case of a dense true covariance matrix via stochastic sparsification which is generally not addressed by sparse PCA approaches. 

\noindent
\underline{\textit{GNNs and VNNs.}} VNNs are graph convolutional neural networks operating on the covariance graph~\cite{sihag2022covariance}, performing an operation that draws analogies with PCA, but achieves improved stability to finite-sample estimation errors -- this has been proved by extending the small perturbation analysis of GNNs \cite{gama2020stability, kenlay2021interpretable, kenlay2021stability, levie2021transferability, maskey2023transferability, gao2023trade} to covariance graphs.
Due to their robustness, VNNs have shown successful performance in brain data processing~\cite{sihag2024explainable}, transferability to high-dimensional data~\cite{sihag2023transferablility}, temporal settings~\cite{cavallo2024stvnn} and biased datasets~\cite{cavallo2024fairvnn}.
All these results, however, hold for the dense covariance matrix which is suboptimal in low-data regimes and computationally heavy. Here, we propose S-VNNs and study their stability when the true covariance matrix is sparse and dense. This merits a different treatment than the perturbation assumption studied in \cite{sihag2022covariance} and overcomes the limitations of dense VNNs. Our stochastically sparsified VNN draws analogies with the stability of GNNs to random link drops~\cite{gao2021stability,gao2021stochastic} but we generalize those findings to the more challenging case where all sparsified probabilities differ rather than being identical. We also propose design strategies for the sparsified probabilities based on our theoretical analysis.

\noindent
\underline{\textit{Graph sparsification in GNNs.}} Graph sparsification techniques have been used for multiple purposes in GNNs but predominantly to improve: (i) scalability in large and dense graphs~\cite{hamilton2018inductive,zhang2019heterogeneous,zeng2021decoupling,peng2022towards,srinivasa2020fast,ye2021sparse}; (ii) GNN expressiveness in graph classification tasks and alleviate oversmoothing~\cite{papp2021dropgnn,Rong2020DropEdge,Fang_2023,zheng2020robust,morris2021power}; and (iii) enhance interpretability~\cite{rathee2021learnt,Li_2023,naber2024mapselect}. 
Our S-VNNs aim on the one hand to improve scalability and on the other to preserve the tractability and theoretical links with covariance-based data processing. 
While our sparsification approaches could be applied also to weighted graphs, we focus on covariance sparsification to characterize theoretically how this affects stability w.r.t. covariance estimation errors and spurious correlations. We also propose a theoretically-grounded stochastic sparsification tailored to sparsity generic dense covariance matrices, which achieves a superior performance than alternatives not developed for this setting.

\section{Conclusion}

In this work, we develop sparse covariance neural networks (S-VNNs) and study the effects of sparsification techniques on covariance neural networks. 
We show that S-VNNs can be more stable to finite-data effects on the sample covariance matrix and more computationally efficient.  
When the true covariance is sparse, we propose S-VNNs with hard/soft thresholding and show that they transfer to the nominal VNN with a rate inversely proportional to the square root of the number of data points. When the true covariance is dense, we put forth a stochastic sparsified approach to reduce the computational complexity while maintaining mathematical tractability and stability. In the framework of stochastic sparsification, we further propose principled strategies, i.e., ACV and RCV, that allow to tune sparsification impact and computation efficiency. Experimental results on both synthetic and real datasets show that the proposed S-VNNs improve substantially the computation time and achieve competitive or better performance compared with nominal VNNs, 
thus validating that downstream performance benefits from sparsification via spurious correlation removal or stochastic regularization effects.
Future work will investigate other potential covariance estimators that can be used in S-VNNs and characterize the regularization effect of stochastic sparsification, which allows the model to observe multiple instances of the sparsified covariance with potential benefits for robustness, expressiveness, generalization and reduced oversmoothing similarly to dropout-like techniques and graph sparsification methods~\cite{papp2021dropgnn,gao2021training,Rong2020DropEdge}. 

\bibliographystyle{IEEEtran}
\bibliography{bibliography}

\clearpage

\appendices

\section{Proof of \Cref{lemma_pca}}
\label{app:pca_dense_stab}

We provide a bound for the stability of PCA, i.e., $\|\mtV^\Tr\vcx-\mthV^\Tr\vcx\|$. We have that 
\begin{align}
\begin{split}
    \left\|\mtV^\Tr\vcx - \mthV^\Tr\vcx\right\| &\leq \|\mtV - \mthV \|\| \vcx \| \\
    &\leq \sqrt{N} \max_i \left\|\vcv_i-\vchv_i\right\|
\end{split}
\end{align}
where we used spectral norm properties.
Using the sin-theta theorem~\cite[Theorem V.3.6]{Stewart90}, we have that
\begin{align}\label{eq:sin_theta_theorem}
    \max_i \left\|\vcv_i-\vchv_i\right\| \leq \frac{\sqrt{2}\|\mtE\|}{\min_j|\lambda_j-\lambda_{j+1}|}
\end{align}
where $\|\mtE\| = \|\mtC - \mthC\|$. 
We now leverage the result from~\cite[Theorem 5.6.1]{vershynin2018high} which shows that $\|\mtE\| \leq \mathcal{O}(t^{-1/2})$ with probability $1-o(1)$, and this concludes the proof.

\qed

\section{Proof of \Cref{cor:sparse_stab_hard,cor:vnn_stability_sparsecov}}
\label{app:proof_sparse_hard}

\subsection{Preliminaries}
We begin by providing a Lemma that will be used in the following.

\begin{lemma}
    \label{lemma:small_eig_perturbation}
    Consider the true and the thresholded sample covariance matrices with their respective eigendecompositions $\mtC = \mtV\mtLambda\mtV^\Tr$ and $\mtbC = \mtbV\bar{\mtLambda}\mtbV^\Tr$ and let $\mtE = \mtC - \mtbC$ be the error estimation matrix such that $\|\mtE\| \ll 1$ (i.e., small perturbation). The following holds:
    \begin{align}
        | \vcv_j^\Tr \vcbv_i | \leq \frac{\| \mtE \|}{|\lambda_i - \lambda_j|} + \mathcal{O}(\|\mtE\|^2)
    \end{align}
\end{lemma}
\begin{proof}
    From \cite[Chapter 2, (10.2)]{Wilkinson_Algebraic},
    using the assumptions in \Cref{cor:sparse_stab_hard,cor:vnn_stability_sparsecov} of $\mtC$ being positive definite with distinct eigenvalues, we have that the first-order approximation of the $i$-th estimated covariance eigenvector $\vctv_i$ is
\begin{equation}\label{eq:perturbation_approx}
    \vcbv_i \approx \vcv_i + \changed{\sum_{k=1, k\neq i}^N \frac{\vcv_k^\Tr\mtE\vcv_i}{\lambda_i - \lambda_k}\vcv_k} + \mtD
\end{equation}
with $\|\mtD\| = \mathcal{O}(\|\mtE\|^2)$. By projecting on $\vcv_j$ and taking the absolute value, we get
\begin{align}
    | \vcv_j^\Tr \vcbv_i | \leq \frac{|\vcv_j^\Tr\mtE\vcv_i|}{|\lambda_i - \lambda_j|} + \mathcal{O}(\|\mtE\|^2) \leq \\
    \frac{\| \mtE \|}{|\lambda_i - \lambda_j|} + \mathcal{O}(\|\mtE\|^2)
\end{align}
    where the last step follows from the definition of spectral norm: $\|\mtE\| = \max_{\vcv, \|\vcv\|_2=1} |\vcv^\Tr \mtE \vcv|$.
\end{proof}

We now provide a general result on VNN stability that will be the starting point to prove \Cref{cor:sparse_stab_hard,cor:vnn_stability_sparsecov}.
\begin{proposition}[Stability of VNNs]
    \label{th:sparse_stab_hard}
    Consider a Lipschitz covariance filter $\mtH(\mtC)$ 
    with constant $P$. 
    Let $\mtbC$ be an estimated covariance matrix of $\mtC$. 
    Then, for any generic signal $\vcx$ where $\|\vcx\|\leq 1$ and estimation error $\mtE = \mtC - \mtbC$ where $\|\mtE\| \ll 1$, it holds that 
    \begin{align}
    \label{eq:stab_bound_hard}
    \|\mtH(\mtbC)\vcx-\mtH(\mtC)\vcx\| \leq P\sqrt{N}\|\mtE\|(1+\sqrt{N}) + \mathcal{O}(\|\mtE\|^2).
    \end{align}
\end{proposition}

\begin{proof}
We will make use of the eigendecompositions $\mtC = \mtV\mtLambda\mtV^\Tr$ and $\mtbC = \mtbV\bar{\mtLambda}\mtbV^\Tr$, where $\mtV$ and $\mtbV$ contain in their columns the eigenvectors of $\mtC$ and $\mtbC$, respectively, and $\mtLambda$ and $\bar{\mtLambda}$ contain their eigenvalues on the diagonal. Following~\cite[eq. (32)-(39)]{sihag2022covariance}, under the assumption of small perturbation, i.e., $\|\mtE\|\ll 1$, the stability bound is the sum of three terms:
\begin{align}
    \mtH(\mtbC)\vcx-\mtH(\mtC)\vcx \approx
    \underbrace{\sum_{i=1}^N\sctx_{i} \sum_{k=0}^K h_{k} \sum_{r=0}^{k-1}\mtC^r\sclambda_i^{k-r-1}(\bar{\lambda}_i-\lambda_i)\vcv_i}_\textnormal{term 1}  \label{eq:t1} \\
    + \underbrace{\sum_{i=1}^N\sctx_{i} \sum_{k=0}^K h_{k} \sum_{r=0}^{k-1}\mtC^r\sclambda_i^{k-r-1}(\lambda_i\mtI_N-\mtC)(\vcbv_i-\vcv_i)}_\textnormal{term 2}  \label{eq:t2} \\
    +\underbrace{\sum_{i=1}^N\sctx_{i} \sum_{k=0}^K h_{k} \sum_{r=0}^{k-1}\mtC^r\sclambda_i^{k-r-1}((\bar{\lambda}_i-\lambda_i)\mtI_N-\mtE)(\vcbv_i-\vcv_i)}_\textnormal{term 3} \label{eq:t3}
\end{align}
where $\sctx_i$ is the $i$-th component of the Covariance Fourier Transform of a generic graph signal $\vcx$, i.e., $\vctx = \mtV^\Tr\vcx$, and we inverted term 1 and 2 compared to~\cite{sihag2022covariance} for ease of explanation in the following. We now analyze the three terms separately.

\textbf{Term 1.}
Leveraging~\cite[eq. (59)-(61)]{sihag2022covariance}, we have that
\begin{align}
    \sum_{i=1}^N\sctx_{i} \sum_{k=0}^K h_{k} \sum_{r=0}^{k-1}\mtC^r\sclambda_i^{k-r-1}(\bar{\lambda}_i-\lambda_i)\vcv_i =  \\\sum_{i=1}^N\sctx_{i} h'(\lambda_i)(\bar{\lambda}_i-\lambda_i)\vcv_i
\end{align}
where $h'(\lambda_i)$ is the derivative of the frequency response of the filter and is bounded in absolute value by $P$.
Using Weyl's theorem \cite[Theorem 8.1.6]{golub13}, we have that $|\bar{\lambda}_i-\lambda_i|\leq \|\mtE\|$. 
Therefore, by taking the norm and using the fact that $\sum_{i=1}^N|\sctx_i|\leq\sqrt{N}\|\vctx\| = \sqrt{N}\|\vcx\| \leq \sqrt{N}$, we get
\begin{align}
    \left\|\sum_{i=1}^N\sctx_{i} h'(\lambda_i)(\bar{\lambda}_i-\lambda_i)\vcv_i\right\| \leq
    \sum_{i=1}^N|\sctx_{i}| |h'(\lambda_i)|\|\mtE\|\|\vcv_i\| \leq \\
     \sum_{i=1}^N|\sctx_{i}| P\|\mtE\| \leq  P\sqrt{N}\|\mtE\|. \label{eq:term2_final1}
\end{align}

\textbf{Term 2}.
Following~\cite[eq. (40)-(54)]{sihag2022covariance}, the norm of term 2 is bounded by
\begin{align}
\label{eq:term1_11}
    \left\|\sum_{i=1}^N\sctx_{i} \sum_{k=0}^K h_{k} \sum_{r=0}^{k-1}\mtC^r\sclambda_i^{k-r-1}(\lambda_i\mtI_N-\mtC)(\vcbv_i-\vcv_i)\right\| \leq \\ \label{eq:term2_11}
    \sqrt{N}\sum_{i=1}^N|\sctx_i| \max_{j}|h(\lambda_i)-h(\lambda_j)||\vcv_j^\Tr\vcbv_i|
\end{align}
where $h(\lambda)$ is the frequency response of the covariance filter. Note that in \eqref{eq:term1_11} we have a term $\sqrt{N}$ that does not appear in~\cite[eq. (54)]{sihag2022covariance} because we consider operator norm instead of uniform norm.

Using \Cref{lemma:small_eig_perturbation} in~\eqref{eq:term2_11}, we get
\begin{align}
    \sqrt{N}\sum_{i=1}^N|\sctx_i| \max_{j\neq i}|h(\lambda_i)-h(\lambda_j)||\vcv_j^\Tr\vcbv_i| \leq \\ 
     \sqrt{N}\|\mtE\| \sum_{i=1}^N|\sctx_i| \max_{j\neq i}\frac{|h(\lambda_i)-h(\lambda_j)|}{|\lambda_i - \lambda_j|} \leq \\
    \sqrt{N}\|\mtE\|\sum_{i=1}^N|\sctx_i| P \leq N\|\mtE\| P
     \label{eq:term1_final1}
\end{align}
where we used the Lipschitz property of the filter and the fact that $\sum_{i=1}^N|\sctx_i|\leq\sqrt{N}\|\vcx\|$.

\textbf{Term 3.}
This term depends on the second-order error $((\bar{\lambda}_i-\lambda_i)\mtI_N-\mtE)(\vcbv_i-\vcv_i)$.
From Weyl's theorem \cite[Theorem 8.1.6]{golub13}, we have that $\|(\bar{\lambda}_i-\lambda_i)\mtI_N-\mtE\| \leq 2\|\mtE\|$.
From \eqref{eq:perturbation_approx}, we have that $\|\vcbv_i-\vcv_i\| = \mathcal{O}(\|\mtE\|)$. Therefore, the product of the two is $\|(\bar{\lambda}_i-\lambda_i)\mtI_N-\mtE\|\|\vcbv_i-\vcv_i\| \approx \mathcal{O}(\|\mtE\|^2)$, which is dominated by the other two terms in \eqref{eq:t1} and \eqref{eq:t2} under small perturbation assumption.

We complete the proof by merging the terms in \eqref{eq:term1_final1},\eqref{eq:term2_final1}.

\end{proof}

\subsection{Proof of \Cref{cor:sparse_stab_hard}}
\label{app:proof_cor_hard}
From~\cite[Theorem 1]{bickel2008covariance}, given a true covariance $\mtC$, a hard-thresholded sample covariance $\mtbC$ and a large enough constant $C$, it holds with probability $1-o(1)$ that
\begin{align}
    \|\mtbC-\mtC\| = \|\mtE\| \leq Cc_0\sqrt{\frac{\log N}{t}}.
\end{align}

By replacing $\|\mtE\|$ in \eqref{eq:stab_bound_hard}, the claim follows.
\qed

\subsection{Proof of \Cref{cor:vnn_stability_sparsecov}}
\label{app:proof_cor_soft}
From~\cite[Theorem 1]{deshp2016sparse}, given a true covariance $\mtC$ and a soft-thresholded sample covariance $\mtbC$, it holds with probability $1-o(1)$ that
\begin{align}
    \|\mtbC-\mtC\| = \|\mtE\| \leq \frac{1}{\sqrt{t}}Cc_0\max(1,\lambda_0)\sqrt{\max\left(\log\frac{N}{c_0^2},1\right)}
\end{align}
for a large enough constant $C$.
By replacing $\|\mtE\|$ in \eqref{eq:stab_bound_hard}, the claim follows.
\qed

\section{Proof of \Cref{prop:pca_stab}}
\label{app:proof_pca_stab}
We have that 
\begin{align}
\begin{split}
\label{eq:initialeqbound}
    \left\|\mtV^\Tr\vcx - \mtbV^\Tr\vcx\right\| &= \left\|\sum_{i=1}^N (\vcv_i-\vcbv_i)^\Tr x_i\right\| \leq \sum_{i=1}^N \|\vcv_i-\vcbv_i\||x_i| \\
    &\le \sum_{i=1}^N \|\vcv_i-\vcbv_i\| \leq N \max_i \left\|\vcv_i-\vcbv_i\right\|.
\end{split}
\end{align}
where we used triangle inequality, the fact that $\|\vcx\|\leq 1$ and the fact that $\|(\vcv_i-\vcbv_i)^\Tr\| = \|\vcv_i-\vcbv_i\|$.
Using the sin-theta theorem~\cite[Theorem V.3.6]{Stewart90} (cf. \eqref{eq:sin_theta_theorem}), we have that $\max_i \left\|\vcv_i-\vcbv_i\right\| \leq \sqrt{2}\|\mtE\|/\min_j|\lambda_j-\lambda_{j+1}|$ where $\|\mtE\| = \|\mtC - \mtbC\|$. Now, we use the expression for $\|\mtE\|$ with hard thresholding in~\cite[Theorem 1]{bickel2008covariance}, i.e.,
\begin{align}
\|\mtE\| \leq c_0\sqrt{\frac{\log N}{t}},
\end{align}
and the claim follows.
\qed

\section{Generalized Covariance Filter Frequency Response}
\label{app:generalized_freq_derivation}

We provide here additional details related to the derivation of the generalized covariance filter frequency response in \Cref{def:generalized_freq_resp} following~\cite{gao2021stability}.
Consider a covariance filter operating on a sequence of $K$ sparsified covariances $\mttC_1\dots\mttC_K$ with eigendecomposition $\mttC_k = \mttV_k \mathbf{\tilde{\mtLambda}}_k \mttV_k^\Tr$ (where $\mttV_k = [\vctv_{k1},\dots,\vctv_{kN}]$ and $\mathbf{\tilde{\mtLambda}}_k = \text{diag}(\sctlambda_{k1},\dots,\sctlambda_{kN})$) according to \Cref{def:cov_filter_random_cov}.
For matrix $\mttC_1$, we can express a signal $\vcx$ as $\vcx = \sum_{i_1=1}^N\schx_{1i_1}\vctv_{1i_1}$, where $\vchx_1 = [\schx_{11},\dots,\schx_{1N}]^\Tr$ is the covariance Fourier transform of $\vcx$ w.r.t. $\mttC_1$.
By performing a graph signal shift, we obtain
\begin{align}
\label{eq:generalized_fr1}
    \vcx^{(1)} = \mttC_1\vcx = \mttC_1\sum_{i_1=1}^N\schx_{1i_1}\vctv_{1i_1} = \sum_{i_1=1}^N\schx_{1i_1}\sctlambda_{1i_1}\vctv_{1i_1}.
\end{align}
When performing a second shift, i.e., $\vcx^{(2)} = \mttC_2\mttC_1\vcx$, we decompose each eigenvector $\vctv_{1i_1}$ by taking its covariance Fourier transform w.r.t. $\mttC_2$, i.e., we get $\vctv_{1i_1} = \sum_{i_2=1}^N\schx_{2i_1i_2}\vctv_{2i_2}$ where $\vctx_{2i_1} = [\schx_{2i_11},\dots,\schx_{2i_1N}]$ is the covariance Fourier transform of $\vchv_{1i_1}$ over $\mttC_2$.
Performing this operation on all eigenvectors $\vctv_{11},\dots,\vctv_{1N}$ and using \eqref{eq:generalized_fr1}, we can write
\begin{align}
\label{eq:generalized_fr2}
    \vcx^{(2)} = \mttC_2\sum_{i_1=1}^N\schx_{1i_1}\sctlambda_{1i_1}\vctv_{1i_1} = \sum_{i_2=1}^N\sum_{i_1=1}^N\schx_{2i_1i_2}\schx_{1i_1}\sctlambda_{2i_2}\sctlambda_{1i_1}\vctv_{2i_2}.
\end{align}
Generalizing this to $k$ shifts, we get
\begin{align}
\label{eq:generalized_fr3}
    \vcx^{(k)} = \sum_{i_k=1}^N\dots\sum_{i_1=1}^N\schx_{ki_{k-1}i_k}\dots\schx_{2i_1i_2}\schx_{1i_1}\prod_{j=1}^k\sctlambda_{ji_j}\vctv_{ki_k}
\end{align}
and by aggregating the $K+1$ shifted signals according to \Cref{def:cov_filter_random_cov} we obtain
\begin{align}
\label{eq:generalized_fr4}
    \vctu = \sum_{i_K=1}^N\dots\sum_{i_1=1}^N\schx_{Ki_{K-1}i_K}\dots\schx_{2i_1i_2}\schx_{1i_1}\sum_{k=1}^Kh_k\prod_{j=1}^k\sctlambda_{ji_j}\vctv_{ki_k}.
\end{align}
From \eqref{eq:generalized_fr4} we see that this representation involves all eigenvalues $\mathbf{\tilde{\mtLambda}}_K,\dots,\mathbf{\tilde{\mtLambda}}_1$ and eigenvectors $\mttV_K,\dots,\mttV_1$ of the sequence of covariance realizations. Therefore, we can consider the coefficients $\{\schx_{1i_1}\}_{i_1=1}^N$ and $\{\schx_{2i_ji_{j+1}}\}_{j=1}^{K-1}$ as the generalized covariance Fourier transform of signal $\vcx$ on the sequence of sparsified covariances $\mttV_K,\dots,\mttV_1$. This supports the definition of a generalized frequency response for a covariance filter over random sparsified covariances in \Cref{def:generalized_freq_resp}.

\section{Proof of \Cref{th:stability_random}}
\label{app:stability_random_proof}

Consider a covariance filter $\mtH(\mttC)\vcx$ operating on the matrix $\mttC = \mtE + \mthC$, which represents a sparsified random matrix that is a copy of $\mthC$ whose elements $c_{ij}$ are dropped with probability $(1-p_{ij})$, where $\mthC$ is the sample covariance matrix and $\mtC$ is the true covariance matrix. Then, the stability of the covariance filter depends on two sources of error. Indeed, by adding and subtracting $\mtH(\mthC)\vcx$ and leveraging the triangle inequality, we obtain
\begin{align}
    \|\mtH(\mtC)\vcx-\mtH(\mttC)\vcx\|^2 = \\
    \|\mtH(\mtC)\vcx - \mtH(\mthC)\vcx + \mtH(\mthC)\vcx - \mtH(\mttC)\vcx\|^2 \leq \\
    \|\mtH(\mtC)\vcx - \mtH(\mthC)\vcx \|^2 + \|\mtH(\mthC)\vcx - \mtH(\mttC)\vcx\|^2 = \\
    \alpha + \beta
\end{align}
where $\alpha$ is an instability error due to the uncertainties in the covariance matrix estimate and $\beta$ depends on the stochastic sparsification of the sample covariance matrix. We analyze and provide an expression for these two terms in the remainder of the proof.

\subsection{Covariance uncertainty error}
The effect of covariance estimation errors on the stability of VNN is deterministic and, consequently, can be bounded by the square of the bound in \Cref{th:cov_filter_stability}: 
\begin{align}
     \|\mtH(\mtC)\vcx - \mtH(\mthC)\vcx \|^2 \leq 
     \frac{P^2k_\textnormal{max}^2}{t^{1-2\epsilon}}\mathcal{O}\left(N + \frac{\|\mtC\|^2\log(Nt)}{k_\textnormal{min}^2t^{4\epsilon}} \right)
\end{align}
with probability at least $1-t^{-2\epsilon}- 2\kappa N/t$ for any $\epsilon \in (0, 1/2]$, where all the terms are defined in \Cref{th:cov_filter_stability}.

\subsection{Covariance sparsification error}

To analyze the stability of VNN to the stochastic sparsification of the covariance matrix, we leverage and extend previous results on the stability of GNNs to stochastic graph perturbations.
We begin by providing some lemmas that will be used in the main statement.

\begin{lemma}
\label{lemma:expectation_square}
    Consider a random covariance matrix $\mttC_r = \mthC + \mtE_r$, where $\mttC_r$ is a copy of the sample covariance matrix $\mthC$ whose elements $\schc_{ij}$ are dropped independently with probability $1-p_{ij}$ and $\mtE_r$ is its distance from $\mthC$. Under the conditions of \Cref{def:stochastic_sparsification}, i.e., edge perturbations are undirected (i.e., $\mtE_r$ is symmetric) and no perturbations occur on the diagonal (i.e., $\mtE_r$ has zeros on the diagonal), it holds that
    \begin{align}
    \label{eq:def_of_Q}
        \trace(\mathbb{E}[\mtE_r^2]) = \sum_{i=1}^N\sum_{n=1}^N \schc_{in}^2(1-p_{in}) = Q.
    \end{align}
\end{lemma}
\begin{proof}
    Each entry of the random error matrix $\mtE_r$ can be represented as $[\mtE_r]_{ij} = -\delta_{ij}c_{ij}$, where $\delta_{ij}$ is a Bernoulli variable that is one with probability $1-p_{ij}$ and zero with probability $p_{ij}$. By performing the matrix multiplication, we can express each element of $\mtE_r^2$ and, consequently, its expectation, as
    \begin{align}
        [\mtE_r^2]_{ij} = \sum_{n=1}^N \schc_{in}\schc_{nj}\delta_{in}\delta_{nj}, \quad [\mathbb{E}[\mtE_r^2]]_{ij} = \sum_{n=1}^N \schc_{in}\schc_{nj}\mathbb{E}[\delta_{in}\delta_{nj}].
    \end{align}
    The Bernoulli variables $\delta_{ij}$ are independent except for $\delta_{ij} = \delta_{ji}$ given the symmetry of $\mtE_r$. Therefore, we have
    \begin{align}
        \mathbb{E}[\delta_{in}\delta_{nj}] = \begin{cases}
                      (1-p_{in})(1-p_{nj})  & \text{if } i \neq j \\
                      (1-p_{in}) & \text{if } i = j
                    \end{cases}
    \end{align}
    We can now compute the trace by summing the elements on the diagonal and using the fact that $\mtE_r$ is symmetric, i.e., 
    \begin{align}
         \trace(\mathbb{E}[\mtE_r^2]) = \sum_{i=1}^N\sum_{n=1}^N \schc_{in}\schc_{ni}\mathbb{E}[\delta_{in}\delta_{ni}] = \sum_{i=1}^N\sum_{n=1}^N \schc_{in}^2(1-p_{in}).
    \end{align}
\end{proof}

\begin{lemma}
\label{lemma:term2stoch}
Consider a covariance filter $\mtH(\mtC)$ with coefficients $\{h_k\}_{k=0}^K$ and generalized integral Lipschitz frequency response with constant $P$. Given some realizations of a random matrix $\mttC_r = \mtE_r + \mthC$, for any signal $\vcx$, it holds that 

{\small
\begin{align}
\label{eq:lemma_term1}
        \mathbb{E}\left[ \sum_{r=1}^K \trace\left( \sum_{k=r}^K\sum_{l=r}^K h_kh_l \mtE_r\mthC^{k+l-2r}\mtE_r\mthC^{r-1}\vcx\vcx^\Tr\mthC^{r-1} \right) \right] \leq \\ P^2\|\vcx\|^2Q
\end{align}
}%
with $Q$ defined in \eqref{eq:def_of_Q}.
\end{lemma}
\begin{proof}
    Following~\cite[eq. (B.13)-(B.15)]{gao2021stability}, we rewrite the term in \eqref{eq:lemma_term1} as
\begin{align}
    \sum_{i=1}^N\schx_i^2\sum_{r=1}^K\trace\left( \sum_{k=r}^K\sum_{l=r}^K h_kh_l\schlambda_i^{2r-2}\mthC^{k+l-2r}\mathbb{E}[\mtE_r^2] \right)
\end{align}
where $\schx_i$ is the $i$-th component of the Covariance Fourier Transform of the signal $\vchx = \mthV^\Tr\vcx$ and $\mthC = \mthV\mathbf{\hat{\Lambda}}\mthV^\Tr$ is the eigendecomposition of $\mthC$ with eigenvectors $\mthV = [\vchv_1,\dots,\vchv_N]^\Tr$ and eigenvalues $\mathbf{\hat{\Lambda}} = \text{diag}(\schlambda_1,\dots,\schlambda_N)$. 
Now, we leverage the property
\begin{align}
\label{eq:trace_property}
    \trace(\mtA\mtB) \leq \frac{\|\mtA+\mtA^\Tr\|}{2}\trace(\mtB)\leq \|\mtA\|\trace(\mtB)
\end{align}
which holds for any square matrix $\mtA$ and positive semi-definite matrix $\mtB$~\cite{wang1986tracebounds}.
We note that $\mathbb{E}[\mtE_r^2]$ is positive semi-definite since $\mtE_r^2$ is the square of a symmetric matrix. Therefore, we use \eqref{eq:trace_property} to write
\begin{align}
    \sum_{i=1}^N\schx_i^2\sum_{r=1}^K\trace\left( \sum_{k=r}^K\sum_{l=r}^K h_kh_l\schlambda_i^{2r-2}\mthC^{k+l-2r}\mathbb{E}[\mtE_r^2] \right) \leq \\
    \sum_{i=1}^N\schx_i^2\left\|\sum_{r=1}^K \sum_{k=r}^K\sum_{l=r}^K h_kh_l\schlambda_i^{2r-2}\mthC^{k+l-2r}\right\|\trace(\mathbb{E}[\mtE_r^2])  \label{eq:lemma_term2_1}
\end{align}
where we also used the linearity of the trace operator to move the summation term.

From~\cite[eq. (B.24)]{gao2021stability}, we have that the term within the norm is bounded by the generalized Lipschitz coefficient of the filter, i.e., 
\begin{align}
    \left\|\sum_{r=1}^K \sum_{k=r}^K\sum_{l=r}^K h_kh_l\schlambda_i^{2r-2}\mthC^{k+l-2r}\right\| \leq P^2.
\end{align}
Then, using \Cref{lemma:expectation_square}, we have that
\begin{align}
\trace(\mathbb{E}[\mtE_r^2]) = \sum_{i=1}^N\sum_{n=1}^N \schc_{in}^2(1-p_{in}) = Q.
\end{align}
Substituting these two terms in~\eqref{eq:lemma_term2_1} leads to
\begin{align}
    \sum_{i=1}^N \schx_i^2P^2Q \leq P^2Q \sum_{i=1}^N \schx_i^2 = P^2Q\|\vcx\|^2,
\end{align}
where we used the fact that $\sum_{i=1}^N \schx_i^2 = \|\vchx\|^2 = \|\mtV^\Tr\vcx\|^2 = \|\vcx\|^2 $, i.e., the covariance Fourier transform does not modify the norm of the signal.
\end{proof}

\begin{lemma}
\label{lemma:distinct_error_trace}
    Consider two distinct random covariance realizations $\mttC_{r_1} = \mthC + \mtE_{r_1}$ and $\mttC_{r_2} = \mthC + \mtE_{r_2}$. Then, the following holds:
    \begin{align}
    \label{eq:trace_crossterms}
        \trace(\mathbb{E}[\mtE_{r_1}\mtE_{r_2}]) = \sum_{i=1}^N\sum_{n=1}^N \schc_{in}^2(1-p_{in})^2 
    \end{align}
\end{lemma}
\begin{proof}
    Similarly to \Cref{lemma:expectation_square}, we can express each element of the expected value of the matrix product $\mtE_{r_1}\mtE_{r_2}$ as 
    \begin{align}
        [\mathbb{E}[\mtE_{r_1}\mtE_{r_2}]]_{ij} = \sum_{n=1}^N \schc_{in}\schc_{nj}\mathbb{E}[\delta_{r_1,in}\delta_{r_2,nj}]
    \end{align}
    where $\delta_{r,ij}$ is a Bernoulli variable relative to realization $\mtE_r$ that is one with probability $1-p_{ij}$ and zero with probability $p_{ij}$. Since $\mtE_{r_1}$ and $\mtE_{r_2}$ are two different realizations, all variables $\delta_{r,ij}$ are independent. Therefore, $\mathbb{E}[\delta_{r_1,in}\delta_{r_2,nj}] = (1-p_{in})(1-p_{nj})$ and the trace becomes
    \begin{align}
        \trace(\mathbb{E}[\mtE_{r_1}\mtE_{r_2}]) = \sum_{i=1}^N\sum_{n=1}^N \schc_{in}^2(1-p_{in})^2.
    \end{align}
\end{proof}

\begin{lemma}
\label{lemma:distinct_error_trace_squared}
    Consider two distinct random covariance realizations $\mttC_{r_1} = \mthC + \mtE_{r_1}$ and $\mttC_{r_2} = \mthC + \mtE_{r_2}$. Then, the following holds:
    \begin{align}
        \trace(\mathbb{E}[\mtE_{r_1}^2\mtE_{r_2}]) = \mathcal{O}((1-p_1)(1-p_2)).
    \end{align}
    where $p_1,p_2$ are two generic probabilities.
\end{lemma}
\begin{proof}
    From \Cref{lemma:expectation_square} we have that 
    \begin{align}
        [\mathbb{E}[\mtE_{r_1}^2]]_{ij} = \sum_{n=1}^N \schc_{in}\schc_{nj}\mathbb{E}[\delta_{r_1,in}\delta_{r_1,nj}].
    \end{align}
    Therefore, the elements of $\mtE_{r_1}^2\mtE_{r_2}$ are
    \begin{align}
        [\mathbb{E}[\mtE_{r_1}^2\mtE_{r_2}]]_{ij} = \sum_{m=1}^N \mathbb{E}[[\mtE_{r_1}^2]_{im}\schc_{mj}\delta_{r_2,mj}] = \\ \sum_{m=1}^N \sum_{n=1}^N \schc_{in}\schc_{nm}\schc_{mj}\mathbb{E}[\delta_{r_1,in}\delta_{r_1,nm}\delta_{r_2,mj}].
    \end{align}
    The three variables $\delta_{r_1,in},\delta_{r_1,nm},\delta_{r_2,mj}$ are all independent with the exception of $\delta_{r_1,in}=\delta_{r_1,ni}$ due to matrix symmetry. 
    Therefore, 
    \begin{align}
        \mathbb{E}[\delta_{r_1,in}\delta_{r_1,nm}\delta_{r_2,mj}] = \\ \begin{cases}
            (1-p_{in})(1-p_{nm})(1-p_{mj}) & \text{if } i \neq m \\
            (1-p_{in})(1-p_{mj}) & \text{if } i = m
        \end{cases}
    \end{align}
    We can now compute the trace:
    \begin{align}
        \trace(\mathbb{E}[\mtE_{r_1}^2\mtE_{r_2}]) = \\\sum_{i=1}^N\sum_{m=1}^N \sum_{n=1}^N \schc_{in}\schc_{nm}\schc_{mi}\mathbb{E}[\delta_{r_1,in}\delta_{r_1,nm}\delta_{r_2,mi}].
    \end{align}
    Therefore, the trace is a sum of terms of quadratic and cubic order in the probability value, i.e., $\mathcal{O}((1-p_1)(1-p_2))$ or $\mathcal{O}((1-p_1)(1-p_2)(1-p_3))$, where $p_1,p_2,p_3$ are generic probabilities values and the quadratic terms dominate the behavior since $1-p<1$. So,
    \begin{align}
        \trace(\mathbb{E}[\mtE_{r_1}^2\mtE_{r_2}]) = \mathcal{O}((1-p_1)(1-p_2)).
    \end{align}
\end{proof}

\begin{lemma}
\label{lemma:expectation_cube}
    Consider the sample covariance matrix $\mthC$ and a random sparsified covariance $\mttC_r = \mthC + \mtE_r$ as in \Cref{def:stochastic_sparsification}. We have that
    \begin{align}
        \trace(\mathbb{E}[\mtE_r^3]) = \mathcal{O}((1-p_1)(1-p_2)).
    \end{align}
    where $p_1,p_2$ are two generic probabilities.
\end{lemma}
\begin{proof}
    Similarly to \Cref{lemma:expectation_square}, we can write the expected value of each element of $\mtE^3_r$ as
    \begin{align}
        [\mathbb{E}[\mtE^3_{r}]]_{ij} = \sum_{m=1}^N \mathbb{E}[[\mtE_{r}^2]_{im}\schc_{mj}\delta_{mj}] = \\ \sum_{m=1}^N \sum_{n=1}^N \schc_{in}\schc_{nm}\schc_{mj}\mathbb{E}[\delta_{in}\delta_{nm}\delta_{mj}]
    \end{align}
    and, consequently, the trace
    \begin{align}
        \label{eq:trace_cubed}
        \trace(\mathbb{E}[\mtE_{r}^3]) = \sum_{i=1}^N\sum_{m=1}^N \sum_{n=1}^N \schc_{in}\schc_{nm}\schc_{mi}\mathbb{E}[\delta_{in}\delta_{nm}\delta_{mi}]
    \end{align} 
    where the Bernoulli variables $\delta_{in},\delta_{nm},\delta_{mi}$ are all independent with the exception of $\delta_{in}=\delta_{ni}$ due to matrix symmetry and the terms where $i=n=m=j$. Specifically, we have that
    \begin{align}
        \mathbb{E}[\delta_{in}\delta_{nm}\delta_{mi}] = \\
        \begin{cases}
            (1-p_{in})(1-p_{nm})(1-p_{mi}) & \text{if } i \neq m \land i \neq n \\
            (1-p_{in})(1-p_{mi}) & \text{if } i = m \oplus i = n \\
            (1-p_{in}) & \text{if } i = m = n 
        \end{cases}
    \end{align}
    where $\oplus$ denotes the xor operator.
    Now we note that, according to \Cref{def:stochastic_sparsification}, the matrix $\mtE_r$ has zeros on the diagonal or, equivalently, $1-p_{ii} = 0 \quad \forall i$. As a consequence, the terms where $i = m = n$ in \eqref{eq:trace_cubed} are all zeros and the trace only contains terms of quadratic and cubic order in the probability value, i.e., $\mathcal{O}((1-p_1)(1-p_2))$ or $\mathcal{O}((1-p_1)(1-p_2)(1-p_3))$, where $p_1,p_2,p_3$ are generic probabilities values and the quadratic terms dominate the behavior since $1-p<1$. Therefore,
    \begin{align}
        \trace(\mathbb{E}[\mtE_r^3]) = \mathcal{O}((1-p_1)(1-p_2)).
    \end{align}
\end{proof}

\textbf{Main statement.}
Let $\vcu=\mtH(\mthC)\vcx$ and $\vctu=\mtH(\mttC)\vcx$ be the outputs of the deterministic and stochastic filter, respectively. We are interested in the term
\begin{align}
    \mathbb{E}[\|\vcu-\vctu\|^2] = \mathbb{E}[\trace(\vcu^\Tr\vcu+\vctu^\Tr\vctu-2\vcu^\Tr\vctu)] =\label{eq:th_rand_u1} \\
    \mathbb{E}[\trace(\vctu^\Tr\vctu - \vcu^\Tr\vcu)]+2\mathbb{E}[\trace(\vcu^\Tr\vcu-\vcu^\Tr\vctu)] \label{eq:th_rand_u2}
\end{align}
where we added and subtracted $\vcu^\Tr\vcu$ and used linearity of expectation and trace.
\changed{We omit the conditioning $\mathbb{E}[\cdot|\mthC]$ for visual clarity.}

We represent a random covariance matrix as $\mttC_r = \mthC + \mtE_r$, where $\mtE_r$ is a random matrix that contains the deviation from the true covariance.
Following~\cite[eq. (B.2)-(B.6)]{gao2021stability}, we express the first term in \eqref{eq:th_rand_u2} as
{\footnotesize
\begin{align}
    \mathbb{E}[\trace(\vctu^\Tr\vctu - \vcu^\Tr\vcu)] = -2\mathbb{E}[\trace(\vcu^\Tr\vcu-\vcu^\Tr\vctu)] + \label{eq:th_rand_dec1} \\
    \sum_{k=1}^K\sum_{l=1}^K
h_kh_l\trace\left( \mathbb{E}\left[\sum_{r=1}^{\min(k,l)}\mthC^{k-r}\mtE_r\mthC^{r-1}\vcx\vcx^\Tr\mthC^{r-1}\mtE_r\mthC^{l-r} \right] \right) + \label{eq:th_rand_dec2} \\
\sum_{k=0}^K\sum_{l=0}^Kh_kh_l\trace(\mathbb{E}[\mtS_{kl}]) \label{eq:th_rand_dec3}
\end{align}
}%
where the term in \eqref{eq:th_rand_dec1} cancels out with the second term in \eqref{eq:th_rand_u2},
the term in \eqref{eq:th_rand_dec2} contains cross-products including two error matrices $\mtE_r$ with the same index $r$, and $\mtS_{kl}$ aggregates the sum of quadratic forms with two error matrices with different indices (i.e., terms of the form
$f_1(\mthC)\mtE_{r_1}f_2(\mthC)\mtE_{r_1}f_3(\mthC)\mtE_{r_2}f_4(\mthC)\mtE_{r_2}f_5(\mthC)$ or $f_1(\mthC)\mtE_{r_1}f_2(\mthC)\mtE_{r_1}f_3(\mthC)\mtE_{r_2}f_4(\mthC)$ for two error matrices $\mtE_{r_1}$ and $\mtE_{r_2}$, $r_1 \neq r_2$). We now proceed to analyze the three terms in \eqref{eq:th_rand_dec1},\eqref{eq:th_rand_dec2},\eqref{eq:th_rand_dec3}.

\textbf{First term.}
The term in \eqref{eq:th_rand_dec1} is the opposite of the second term in \eqref{eq:th_rand_u2}, so it cancels out when substituted.

\textbf{Second term.}
Analogously to~\cite[eq. (B.8)]{gao2021stability}, we rewrite the term in \eqref{eq:th_rand_dec2} leveraging the linearity of trace and expectation, the trace cyclic property $\trace(\mtA\mtB\mtC) = \trace(\mtC\mtA\mtB) = \trace(\mtB\mtC\mtA)$ and rearranging the terms to change the sum limits. Then we exploit \Cref{lemma:term2stoch} to upper bound it as
{\small
\begin{align}
    \mathbb{E}\left[ \sum_{r=1}^K \trace\left( \sum_{k=r}^K\sum_{l=r}^K h_kh_l \mtE_r\mthC^{k+l-2r}\mtE_r\mthC^{r-1}\vcx\vcx^\Tr\mthC^{r-1} \right) \right] \leq \\ P^2\|\vcx\|^2_2Q
\end{align}
}%
where $Q$ is defined in \Cref{lemma:term2stoch} and is a sum of terms linear in the probability value.

\textbf{Third term.}
The term in \eqref{eq:th_rand_dec3} can be bounded by an expression similar to~\eqref{eq:lemma_term2_1}, but with at least two of the terms among $\{\mathbb{E}[\mtE_{r_1}^2],\mathbb{E}[\mtE_{r_2}^2],\mathbb{E}[\mtE_{r_1}],\mathbb{E}[\mtE_{r_2}]\}$ within the trace operator. Using \Cref{lemma:distinct_error_trace,lemma:distinct_error_trace_squared,lemma:expectation_cube}, we know that these trace terms are of the order $\mathcal{O}((1-p_1)(1-p_2))$ for two generic probability values $p_1,p_2$.
Since the frequency response of the covariance filter $h(\vclambda)$ is bounded, and consequently the coefficients $h_k$ are also bounded, we have that
\begin{align}
    \sum_{k=0}^K\sum_{l=0}^Kh_kh_l\trace(\mathbb{E}[\mtS_{kl}]) = \mathcal{O}((1-p_1)(1-p_2)).
\end{align}
By substituting the three bounds into~\eqref{eq:th_rand_u1}, noticing that the terms of quadratic order $\mathcal{O}((1-p_1)(1-p_2))$ are dominated by the linear terms in $Q$ and using the fact that $\|\vcx\| \leq 1$, we obtain the final bound:
\begin{align}
    \mathbb{E}[\|\vcu-\vctu\|^2] \leq P^2Q + \mathcal{O}((1-p_1)(1-p_2)).
\end{align}

\qed

\end{document}